\documentclass[11pt]{article} 

\usepackage{arxiv_preprint}

\usepackage{amsmath,amsfonts,bm}

\def\eqref#1{equation~\ref{#1}}

\def\1{\bm{1}}

\DeclareMathAlphabet{\mathsfit}{\encodingdefault}{\sfdefault}{m}{sl}
\SetMathAlphabet{\mathsfit}{bold}{\encodingdefault}{\sfdefault}{bx}{n}

\newcommand{\R}{\mathbb{R}}

\newcommand{\NS}{\mathcal{N}_5}
\newcommand{\norm}[1]{\lVert#1\rVert}

\newcommand{\Compass}{ExpertMuon-Compass}

\usepackage[T1]{fontenc}
\usepackage{fix-cm}
\usepackage{amsmath,amssymb,booktabs,array,graphicx,mathtools,longtable}
\renewcommand{\eqref}[1]{\textup{(\ref{#1})}}
\usepackage{float,algpseudocode,placeins}
\floatstyle{ruled}
\newfloat{algorithm}{tbp}{loa}
\floatname{algorithm}{Algorithm}
\usepackage{tikz}
\usetikzlibrary{arrows.meta,positioning,calc}
\usepackage{xcolor}
\usepackage{amsthm}
\newtheorem{compasstheorem}{Theorem}
\newtheorem{compasslemma}{Lemma}[section]
\newtheorem{compassproposition}{Proposition}[section]

\theoremstyle{definition}

\usepackage{hyperref,url}

\title{ExpertMuon-Compass: Alignment-Guided\\Step Sizes for Mixture-of-Experts Training}

\author{%
Omatharv Bharat Vaidya, Ashwin Vinod\textsuperscript{*}, Pedram Akbarian\textsuperscript{*},\\
Aditya Sai Ellendula\textsuperscript{*}, Connor T. Jerzak, Nhat Ho\\[4pt]
{\small University of Texas at Austin}\\[2pt]
{\small \textsuperscript{*}Equal contribution}}

\usepackage{bookmark}
\renewcommand{\arxivshorttitle}{ExpertMuon-Compass}
\hypersetup{pdftitle={ExpertMuon-Compass: Alignment-Guided Step Sizes for Mixture-of-Experts Training},
  pdfauthor={Omatharv Bharat Vaidya, Ashwin Vinod, Pedram Akbarian, Aditya Sai Ellendula, Connor T. Jerzak, Nhat Ho}}

\begin{document}

\maketitle

\begin{abstract}
Mixture-of-experts (MoE) language models send each token to a few experts, so each expert is trained on a different part of the data, and this part changes during training. With a shared learning rate, Muon applies updates of roughly the same size to expert matrices of the same shape, even when an expert's update is poorly aligned with its current gradient. We here propose ExpertMuon-Compass (Compass), which multiplies the Muon step of each expert by two factors. A family factor compares the cosine between the expert's orthogonalized update and its gradient with the same cosine for the other experts in its layer. A scalar radius aggregates the alignment between corresponding rows of the update and gradient into one step-length multiplier. Compass keeps the update direction and the momentum buffer of Muon. In pretraining on FineWeb-Edu, Compass with Nesterov momentum on all matrices performs as well as or better than Muon, NorMuon, and other optimizers, with weight decay matched to NorMuon in the longer runs. Adding its factors to NorMuon gives the same or a lower loss than NorMuon. Compass is the most effective when the data seen by each expert varies during training, for example, as when the languages of a multilingual corpus arrive in separate blocks. With Compass, the expert load stays balanced, and the router assigns tokens to experts more decisively. We also prove a perturbation bound for the two factors.
\end{abstract}

\section{Introduction}

Mixture-of-experts (MoE) language models increase capacity by activating a small subset of expert networks for each token \citep{shazeer2017moe,fedus2022switch}. The router also decides what an expert learns from, since the expert receives gradients only from the tokens assigned to it, and these assignments shift during training. Muon \citep{jordan2024muon} is now used to train large MoE models \citep{liu2025muonscalable}. It orthogonalizes the momentum of a weight matrix before the update, so the update has a nearly fixed norm. With a shared learning rate, the expert matrices then move by about the same amount at each step, even when an update agrees only weakly with the current gradient. Weak alignment can occur when the tokens routed to an expert change while its momentum still reflects earlier tokens.

Current MoE training handles the experts through the router and the learning rate. Load balancing adjusts the router with an auxiliary loss or with bias updates \citep{fedus2022switch,wang2024lossfree}. The learning rate is usually shared across all experts \citep{muennighoff2024olmoe,liu2025muonscalable}, and Excitation scales an expert's update by how often it is used \citep{shaier2026excitation}. In these approaches, the step size of an expert is independent of how well its update agrees with its current gradient. We propose \Compass{} (Compass), which sets the step of each expert based on this agreement, measured on the orthogonalized update that Muon applies. Compass computes the cosine between the update and the gradient of an expert matrix and compares it with the cosines of the other experts in the layer. An expert with a below-average cosine takes a smaller step than an expert with an above-average cosine, and the factors average to one over the experts. A second factor, the scalar radius, aggregates the alignment between corresponding rows of the update and gradient into a multiplier that can shorten the step while preserving its direction. Both factors are scalars. Compass thus keeps the update direction and the single momentum buffer of Muon, and the factors need only inner products and row norms. They can also be added to other Muon-based updates such as NorMuon \citep{li2025normuon}.

Changing the data that each expert sees during training gives a direct test of when expert step sizes matter under distribution shift (as occurs, e.g., in continual learning). Hence, we train an eight-language MoE model on the same token batches in two orders, once shuffled and once grouped into repeated single-language blocks. With shuffled batches, Compass and Nesterov Muon reach nearly the same test loss. The blocked order changes the data seen by the experts at each language switch, and here Compass reaches a lower test loss than Nesterov Muon in all eight languages (Figure~\ref{fig:order-effect}). Adding the Compass factors to the update of NorMuon also lowers its test loss in every language.

\begin{figure}[t]
\centering
\includegraphics[width=\linewidth]{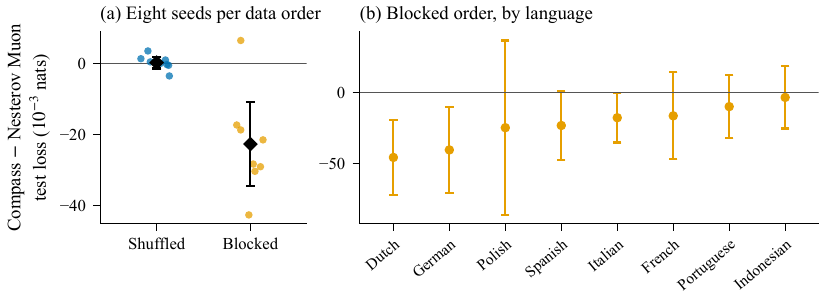}
\caption{\textbf{Compass under language-blocked data order.} Final test loss of Compass minus Nesterov Muon on an eight-language corpus. (a) Each point is one seed; diamonds and bars show the mean and 95\% confidence interval over eight seeds. (b) Mean difference and 95\% confidence interval for each language under blocked order, in the same units. Compass uses Nesterov momentum for expert matrices and heavy-ball momentum for attention and router matrices, and learning rates are tuned separately for each method and data order (Section~\ref{sec:order-results}).}
\label{fig:order-effect}
\end{figure}

Our contributions are as follows. (1) Compass, a step-size rule for the expert matrices of Muon that uses the cosine between each expert's orthogonalized update and its gradient and keeps the update direction. (2) Pretraining experiments on FineWeb-Edu with training budgets from 8.2M to 1.05B tokens. At the shortest budget, we compare Compass with Nesterov Muon, NorMuon, heavy-ball Muon, Excitation, Magma, and AdamW \citep{jordan2024muon,li2025normuon,shaier2026excitation,joo2026magma,loshchilov2019adamw}; at longer budgets, we compare it with Nesterov Muon and NorMuon. We tune the learning rate of every method in the same way, and our ablations remove each factor or replace both factors with a constant expert scale. With Nesterov momentum on all matrices, and with weight decay matched to NorMuon in the longer runs, Compass performs as well as or better than these optimizers, and adding the Compass factors to NorMuon gives the same or a lower loss than NorMuon. (3) An analysis of training. Compass keeps the expert load balanced and makes the router assign tokens to experts more decisively. In the language-blocked experiment, its factors spread out earliest in each block, soon after the data seen by the experts change. (4) A bound on how much the two factors can change the update when the gradients change.
\section{Related work}

We review matrix optimizers, MoE training, and work on alignment and expert updates.

\paragraph{Matrix optimizers.}
Muon orthogonalizes the momentum of each weight matrix with a Newton-Schulz iteration \citep{jordan2024muon}. Its reference implementation uses Nesterov momentum \citep{jordan2024muon}, and in large models its behavior depends on how the update is scaled, how parameters are grouped, and the numerical precision \citep{liu2025muonscalable,khona2026soapmuon}. NorMuon adds row-wise second-moment normalization to the orthogonalized update \citep{li2025normuon}. SOAP runs Adam in the eigenbasis of Shampoo's preconditioner \citep{vyas2024soap}. Compass keeps the single momentum buffer of Muon and adds scalar factors for expert matrices.

\paragraph{Training mixture-of-experts models.}
Work on training MoE models has focused mostly on the router. MoE training usually balances expert load with an auxiliary loss \citep{fedus2022switch} or with router bias updates \citep{wang2024lossfree}, and our models use the bias updates. OLMoE is trained with AdamW \citep{muennighoff2024olmoe}, and Moonlight trains an MoE model with Muon \citep{liu2025muonscalable}. Load-balancing methods act on the router, and OLMoE and Moonlight use one learning rate for all experts. Compass sets the step of each expert from its own update and gradient.

\paragraph{Alignment and expert updates.}
Alignment between an update and the gradient has also been used to modify the update. Cautious optimizers use the sign alignment between an update and the gradient to mask individual coordinates \citep{liang2024cautious}. Magma combines block-level alignment with stochastic masking \citep{joo2026magma}. Excitation sets expert update multipliers from expert utilization \citep{shaier2026excitation}. OMoE encourages diverse expert representations through orthogonal optimization \citep{liu2024omoe}. Compass measures alignment after orthogonalization and compares an expert with the matching matrices of the other experts in its layer. It also combines the row cosines of a matrix into a single scalar, preserving the direction computed by Muon.

\section{ExpertMuon-Compass}
\label{sec:method}

Compass keeps the Muon direction of each expert matrix and sets its length using two factors computed from alignment, and Section~\ref{sec:compass-theory} bounds how much these factors can change the update.

\subsection{Update direction and parameter groups}

We group the matching projection matrices of all experts in one layer into a \emph{family}; for example, the first projection of every expert in a layer forms one family. Expert $e$ has weights $W_{e,t}\in\R^{m\times n}$ and clipped gradient $G_{e,t}$, and rows index the output dimension. With $\beta=0.95$ and $M_{e,0}=0$, the momentum and the Nesterov input are
\begin{equation}
 M_{e,t}=\beta M_{e,t-1}+G_{e,t},\qquad
 D_{e,t}=G_{e,t}+\beta M_{e,t}.
 \label{eq:momentum}
\end{equation}
The Nesterov input uses the updated buffer. Five Newton-Schulz iterations with the coefficients of \citet{jordan2024muon} give the orthogonalized direction $O_{e,t}=\NS(D_{e,t})$. In exact arithmetic, they keep the singular vectors of $D_{e,t}$ and bring its singular values closer together (Appendix~\ref{sec:extended-arithmetic}). As in \citet{liu2025muonscalable}, we multiply this direction by the shape factor $s=0.2\sqrt{\max(m,n)}$, which makes the size of the update close to that of an AdamW update, and write $U_{e,t}=sO_{e,t}$.

Compass uses the Nesterov input for expert matrices. For attention and router matrices, it orthogonalizes either the heavy-ball momentum $M_t$ (written N/H) or the Nesterov input $G_t+\beta M_t$ (written N/N); in both cases, the family factor $c$ and the scalar radius $\alpha$ defined below are the same. The N/N configuration matches the momentum convention of the Nesterov Muon baseline. We use AdamW for embeddings, the untied output head, biases, and normalization parameters.

\subsection{Expert and row scaling}

We drop the step index $t$ of the matrices and scores in this subsection. For arrays of equal shape, define
\begin{equation}
 \operatorname{cos}(A,B)=
 \frac{\langle A,B\rangle_F}{\norm A_F\norm B_F}.
 \label{eq:cosine}
\end{equation}
We measure alignment on the unscaled direction $O_e$. The score of expert $e$ is $a_e=[\operatorname{cos}(O_e,G_e)]_+$, where $[x]_+=\max(0,x)$, and a bar denotes the mean over the family. To avoid division by zero, we add small positive constants to the scores and denominators (Appendix~\ref{sec:numerics}). The family factor is
\begin{equation}
 u_e=a_e/\overline a,\qquad
 v_e=\min(1,u_e),\qquad
 c_e=v_e/\overline v.
 \label{eq:expertgate}
\end{equation}
The first normalization compares an expert with its family, and the cap limits this ratio at one. The second normalization makes the family multipliers average to one, setting their relative scale without fixing the total norm of the expert updates. Capping gives equal preliminary weights to above-average alignment scores, limiting the influence of unusually large scores on relative expert step sizes.

The second factor works on the rows of each expert matrix. Each row $r$ receives the score $b_{er}=[\operatorname{cos}(O_{er},G_{er})]_+$ in the same way. With $\overline b_e$ the mean over the rows of expert $e$, define
\begin{equation}
 q_{er}=\min(1,b_{er}/\overline b_e),\qquad
 h_{er}=1+\lambda_t(q_{er}-1),
 \label{eq:rowgate}
\end{equation}
where $\lambda_t=\eta_t/\eta_{\max}$ is the current learning rate divided by its peak value, so the row factors move toward one as the learning rate decays. The scalar radius is the norm of the row-weighted matrix divided by the norm of $O_e$,
\begin{equation}
 \alpha_e=\sqrt{\frac{\sum_rh_{er}^2\norm{O_{er}}_2^2}
 {\norm{O_e}_F^2}},\qquad
 V_e=c_e\alpha_eU_e,
 \label{eq:radius}
\end{equation}
so that
\begin{equation}
 \norm{\alpha_eO_e}_F
 =\norm{\operatorname{diag}(h_e)O_e}_F.
 \label{eq:normmatch}
\end{equation}
Weighting the rows directly would change the direction of $O_e$. The scalar $\alpha_e$ keeps the direction and takes only the length of the row-weighted matrix. It lies between zero and one. We call $c_e\alpha_e$ the expert scale, and it can exceed one because $c_e$ can. Figure~\ref{fig:method} shows the whole update for one expert matrix.

\begin{figure}[t]
\centering
\begin{tikzpicture}[x=1cm,y=1cm,>=Latex,
box/.style={draw=black!65,rounded corners=2pt,align=center,font=\scriptsize,minimum height=.85cm,fill=blue!4},
arrow/.style={->,thick,draw=black!65}]
\node[box,text width=2.15cm] (grad) at (0,0) {Router-selected tokens\\Backward pass and clipping\\Gradient $G_e$};
\node[box,text width=2.5cm] (dir) at (3.15,0) {Expert Nesterov\\$\mathcal{N}_5(G_e+\beta M_e)$\\Matrix direction $O_e$};
\node[box,text width=2.6cm,fill=orange!8] (whole) at (6.5,.82) {Matrix cosine with $G_e$\\Family factor $c_e$};
\node[box,text width=2.6cm,fill=teal!8] (row) at (6.5,-.82) {Row cosines with $G_e$\\Scalar radius $\alpha_e$};
\node[box,text width=2.5cm,fill=teal!12] (update) at (10.1,0) {One scalar per matrix\\$V_e=s c_e\alpha_e O_e$\\Decoupled weight decay};
\draw[arrow] (grad)--(dir);
\draw[arrow] (dir.east)--(4.75,0)|-(whole.west);
\draw[arrow] (dir.east)--(4.75,0)|-(row.west);
\draw[arrow] (whole.east)--(8.4,.82)|-(update.west);
\draw[arrow] (row.east)--(8.4,-.82)|-(update.west);
\draw[arrow] (dir.north)--(3.15,1.65)--node[above,font=\scriptsize] {Preserve the computed direction $O_e$}(10.1,1.65)--(update.north);
\end{tikzpicture}
\caption{\textbf{The Compass update for one expert matrix.} Alignment is measured on the orthogonalized direction $O_e$. The family factor $c_e$ compares whole matrices within a family, and the scalar radius $\alpha_e$ summarizes row-level alignment. Both scale the same direction. Router load balancing is handled separately by bias updates.}
\label{fig:method}
\end{figure}
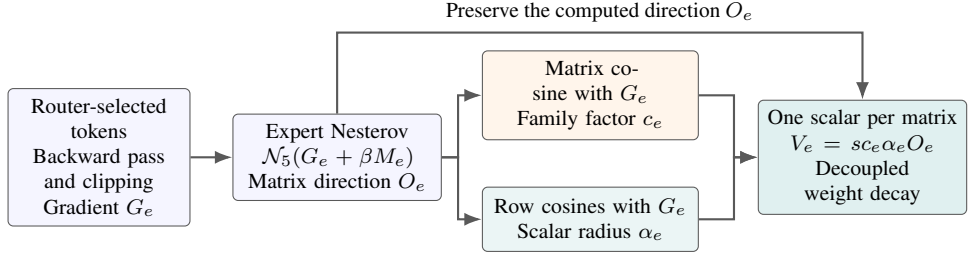

\subsection{Complete update}

With learning rate $\eta_t$ and decoupled weight decay $\omega$, each expert matrix is updated as
\begin{equation}
 W_{e,t+1}=(1-\eta_t\omega)W_{e,t}-\eta_t V_{e,t}.
 \label{eq:update}
\end{equation}
Experts that receive zero tokens in a step are excluded from the family means and keep their weights and momentum. Algorithm~\ref{alg:compass} in Appendix~\ref{sec:new-details} summarizes the update, and Appendix~\ref{sec:numerics} gives the constants and the FP32 order of operations.

\subsection{Perturbation bound for the two factors}
\label{sec:compass-theory}

Both factors divide a score by a mean score, and a small mean can amplify errors in the scores. We therefore bound how much the update at one step changes when the directions and the scores change, for example, when the gradients come from a different batch. We drop the step index and assume exact arithmetic, with the small positive constants of Appendix~\ref{sec:numerics} added to the scores. The family factor of Equation~\eqref{eq:expertgate} equals $c_e=\min(a_e,\overline a)/\mu$, where $\mu=\overline a\,\overline v$ is the family mean of the capped scores $\min(a_e,\overline a)$, and the row factors divide the row scores by their means $\nu_e=\overline b_e$. A tilde marks the quantities computed from a second family of directions and gradients.

\begin{compasstheorem}[Perturbation bound for the family factor and radius]
\label{thm:cth-stability}
Let $V_e$ and $\widetilde V_e$ be the updates of Equation~\eqref{eq:radius} for a family of $E$ matrices, computed from $(O_e,G_e)$ and $(\widetilde O_e,\widetilde G_e)$ with the same shape factor $s$ and schedule multiplier $\lambda\in[0,1]$. Let
\[
 \delta_O=\max_e\|O_e-\widetilde O_e\|_F,\qquad
 \delta_a=\max_e|a_e-\widetilde a_e|,\qquad
 \delta_b=\max_{e,r}|b_{er}-\widetilde b_{er}|,\qquad
 R=\max_e\|\widetilde O_e\|_F,
\]
$A=\min\{\mu,\widetilde\mu\}$ and $B=\min_e\min\{\nu_e,\widetilde\nu_e\}$. Then $A,B>0$ and
\begin{equation}
 \frac1E\sum_e\|V_e-\widetilde V_e\|_F
 \leq s\left(3\delta_O+\frac{2R\delta_a}{A}
                 +\frac{2\lambda R\delta_b}{B}\right).
 \label{eq:cth-stability}
\end{equation}
\end{compasstheorem}

Because $A$ is the smaller of two means over the experts, a few experts with scores near zero have a small effect on it. Similarly, $B$ is the smallest of the row means $\nu_e$ and $\widetilde\nu_e$, and a few rows with scores near zero have a small effect on each of these means. The row term is proportional to $\lambda$, so it vanishes at the end of training, when the row factors equal one. The bound also holds with the constants of Appendix~\ref{sec:numerics}, for zero matrices and for scores at the caps, and it allows the scores to depend on the same gradient as the direction. We prove it in Appendix~\ref{sec:cth-proofs}, where we also convert gradient errors into score errors, and Appendix~\ref{sec:properties} compares the scalar radius with row weighting at equal step length. \citet{kim2026newtonschulz} and \citet{shulgin2025inexact} analyze finite Newton-Schulz iterations and inexact Muon updates; our analysis covers the family normalization and the scalar radius.

\section{Experiments}
\label{sec:experiments}

We compare Compass with the baselines at 8.2M tokens (Section~\ref{sec:short-runs}), in longer training where weight decay matters (Section~\ref{sec:long-context}), and under a language-blocked data order (Section~\ref{sec:order-results}).

\subsection{Experimental setup}

We pretrain MoE language models from random initialization on FineWeb-Edu and FineWeb2 \citep{penedo2024fineweb,penedo2025fineweb2}. The models have six pre-norm Transformer layers, GELU experts, and top-2 routing, and expert load is balanced with auxiliary-loss-free router bias updates \citep{wang2024lossfree}. We keep parameters, gradients, optimizer states, and matrix products in FP32 for every method, so that the comparison reflects the optimizer updates.

On FineWeb-Edu, we use widths 1,024 and 2,048 with eight experts per layer (329.8M and 1,112.5M total parameters). The short runs train for 2,000 steps with batch size 16 and context length 256, for 8.2M tokens. The long runs use width 1,024, context length 1,024, and 4,000 steps (65.5M tokens, 330.6M parameters), and we extend them to 16,000 and 64,000 steps (262M and 1.05B tokens). Training sequences are sampled from a 57M-token training split for runs up to 65.5M tokens and from 380M- and 1.14B-token splits for the two longer budgets.

The matrix and AdamW learning rates follow the cosine schedule
\begin{equation}
 \eta_t=\eta_{\max}\lambda_t,\qquad
 \lambda_t=\tfrac12\{1+\cos(\pi(t-1)/(T-1))\},\qquad t=1,\ldots,T,
 \label{eq:schedule}
\end{equation}
with zero warmup. In the runs from random initialization, the parameters that the Muon-based methods update with AdamW use a peak learning rate of 0.003, $(\beta_1,\beta_2)=(0.9,0.95)$, and weight decay of 0.01. Unless stated otherwise, the matrix weight decay is $\omega=0.01$.

For each method, we tune the learning rate over three values on two tuning seeds and select the value with the lowest mean final validation loss. We then train eight evaluation seeds at the selected rate and report their mean final validation loss. All methods use the same random seeds, which set both the initialization and the order of the training batches. We therefore compare the methods seed by seed, with 95\% confidence intervals (CIs) from paired $t$-tests (Appendix~\ref{sec:new-statistics}).

\paragraph{Baselines and ablations.}
Nesterov Muon orthogonalizes $G+0.95M$ for all hidden matrices, and heavy-ball Muon orthogonalizes $M$. They share the shape factor and AdamW settings of Compass. For NorMuon, we use the reference implementation of \citet{li2025normuon}, with its own row normalization and update scaling, in FP32. We implemented Excitation from its published update rule on top of heavy-ball Muon. For Magma, we use a public AdamW-based implementation. The AdamW baseline updates all parameters with its own tuned learning rate. The ablations are expert-only Nesterov (Compass with $c=\alpha=1$); a constant expert scale, which multiplies every expert step by the mean of $c\alpha$ from a separate Compass run and tunes only its learning rate; and Compass with only the family factor or only the radius. The constant-scale control tests whether alignment-dependent variation in expert step sizes improves on a fixed multiplier matched to their average scale. Computing the Compass factors adds 4 to 7\% to the wall-clock time of a training run (Table~\ref{tab:timing} in Appendix~\ref{sec:new-details}). Learning-rate grids and the remaining settings are in Appendix~\ref{sec:new-details}, and Appendix~\ref{sec:extended-setup} applies Compass to JetMoE-8B, a released MoE model with about 8B parameters \citep{shen2024jetmoe}.

\subsection{Pretraining at 8.2M tokens}
\label{sec:short-runs}
\begin{table}[t]
\centering
\caption{Pretraining on FineWeb-Edu at 8.2M tokens (context length 256). Final validation loss (nats), mean $\pm$ standard deviation over eight seeds. The momentum column gives the choice for expert matrices and for attention and router matrices (N: Nesterov, H: heavy-ball); Adam marks the AdamW-based methods. The best mean in each column is in bold.}
\label{tab:current-width}
\small
\setlength{\tabcolsep}{4pt}
\begin{tabular}{@{}lcrr@{}}
\toprule
Method & Momentum & Width 1,024 & Width 2,048 \\
\midrule
Compass & N/H & $5.45397\pm0.00710$ & $5.32207\pm0.01184$ \\
Compass & N/N & $\mathbf{5.44908}\pm0.00768$ & $5.32340\pm0.01162$ \\
Nesterov Muon & N/N & $5.45252\pm0.00861$ & $5.32296\pm0.00973$ \\
NorMuon & N/N & $5.45211\pm0.00619$ & $5.32626\pm0.00909$ \\
Heavy-ball Muon & H/H & $5.46876\pm0.00673$ & $5.32991\pm0.01169$ \\
Excitation & H/H & $5.46430\pm0.00579$ & $5.32789\pm0.01066$ \\
Magma & Adam & $5.85811\pm0.01501$ & $5.82693\pm0.04095$ \\
AdamW & Adam & $6.03952\pm0.01077$ & $5.93637\pm0.02066$ \\
\midrule
\multicolumn{4}{@{}l}{\textit{Ablations of Compass (N/H)}}\\
Expert-only Nesterov & N/H & $5.46000\pm0.00322$ & $5.32783\pm0.01330$ \\
Constant expert scale & N/H & $5.46146\pm0.00635$ & $5.32505\pm0.01347$ \\
Family factor only & N/H & $5.45695\pm0.00797$ & $5.32358\pm0.01078$ \\
Radius only & N/H & $5.45528\pm0.00548$ & $\mathbf{5.32186}\pm0.01198$ \\
\bottomrule
\end{tabular}
\vspace{-4pt}
\end{table}

Table~\ref{tab:current-width} summarizes the short runs, and Figure~\ref{fig:paired-current} and Table~\ref{tab:current-paired} in Appendix~\ref{sec:new-statistics} show the paired differences. With Nesterov momentum on all matrices (N/N), Compass has the lowest mean loss at width 1,024, below NorMuon on six of eight seeds (paired permutation test $p=0.031$). At width 2,048, it has a lower mean loss than NorMuon and matches Nesterov Muon. With heavy-ball momentum on attention and router matrices (N/H), Compass has a lower loss than heavy-ball Muon on all 16 seeds. An earlier comparison with heavy-ball Muon at these widths, with separate learning-rate sweeps, is in Appendix~\ref{sec:width}. Excitation, Magma, and AdamW have higher losses than either Compass configuration. At this budget, Compass matches or improves on each baseline.

The ablations test whether this comes from the expert factors (Table~\ref{tab:current-attribution} in Appendix~\ref{sec:new-statistics}). With the momentum settings held fixed, Compass has a lower loss than the constant expert scale at width 1,024 on all eight seeds. Compass also has a lower mean loss than expert-only Nesterov at both widths. At width 1,024 it is below the family-factor-only ablation on all eight seeds, and at width 2,048 it matches the radius-only ablation. At width 1,024, alignment-dependent expert steps thus improve on the constant expert scale; Appendix~\ref{sec:earlier-alignment} gives more ablations up to width 2,560.

\subsection{Longer training at context length 1,024}
\label{sec:long-context}

Longer training makes the weight decay more important in the comparison with NorMuon. With its selected learning rate of 0.016 and weight decay 0.01, NorMuon shrinks each weight by a factor of $1-1.6\times10^{-4}$ per step at the peak learning rate. This per-step decay is 6.4 to 7.1 times that of Compass and Nesterov Muon at their selected rates. We therefore also train Compass and Nesterov Muon with the same per-step decay as NorMuon, setting $\omega$ so that $\eta_t\omega$ matches, and tune both over the rates $\{0.00125,0.00175,0.0025\}$. For Compass, we also evaluate twice the decay.

\begin{table}[t]
\centering
\caption{Pretraining at context length 1,024 (width 1,024, 65.5M tokens). Final validation loss (nats), mean $\pm$ standard deviation over eight seeds, and the selected learning rate. The upper block uses weight decay $\omega=0.01$. The lower block gives every method the per-step decay of NorMuon (learning rate 0.016, $\omega=0.01$), $\eta_t\omega=1.6\times10^{-4}\lambda_t$; the rows marked twice the decay use twice this per-step decay. The best mean is in bold.}
\label{tab:current-long}
\small
\setlength{\tabcolsep}{4pt}
\begin{tabular}{@{}lcrr@{}}
\toprule
Method & Momentum & Validation loss & Learning rate \\
\midrule
Compass & N/H & $4.35013\pm0.00837$ & 0.0025 \\
Expert-only Nesterov & N/H & $4.35145\pm0.00875$ & 0.0025 \\
Constant expert scale & N/H & $4.34941\pm0.00849$ & 0.0025 \\
Compass & N/N & $4.31807\pm0.00516$ & 0.0025 \\
Nesterov Muon & N/N & $4.31872\pm0.00451$ & 0.00225 \\
\midrule
\multicolumn{4}{@{}l}{\textit{NorMuon's per-step weight decay}}\\
NorMuon & N/N & $4.30685\pm0.00514$ & 0.016 \\
Nesterov Muon & N/N & $4.30257\pm0.00850$ & 0.00175 \\
Compass on the NorMuon update & N/N & $4.30332\pm0.00505$ & 0.016 \\
Compass & N/N & $4.30225\pm0.00407$ & 0.00175 \\
NorMuon, twice the decay & N/N & $4.29312\pm0.00762$ & 0.016 \\
Nesterov Muon, twice the decay & N/N & $4.28791\pm0.00760$ & 0.00175 \\
Compass, twice the decay & N/N & $\mathbf{4.28522}\pm0.00449$ & 0.00175 \\
\bottomrule
\end{tabular}
\vspace{-4pt}
\end{table}

In Table~\ref{tab:current-long}, at the default decay, Nesterov momentum on the attention and router matrices lowers the loss of Compass, which then matches Nesterov Muon, and the two again reach nearly the same loss with NorMuon's per-step decay. With twice that decay, Compass has the lowest mean loss in the table, also when NorMuon and Nesterov Muon use twice the decay with the same tuning. On more seeds at twice the decay, chosen before training, Compass stays below NorMuon (lower on 22 of 32 seeds) and slightly below Nesterov Muon (30 of 46), and Compass on the NorMuon update stays below NorMuon (20 of 32). On the full 3M-token validation split, all three differences are significant, and Compass has the lowest WikiText-2 test loss of the models we scored (Appendix~\ref{sec:new-statistics}). With matched weight decay, Compass therefore matches or improves on each baseline at this budget.

\begin{table}[t]
\centering
\caption{Compass against NorMuon as the training budget grows (width 1,024, context length 1,024). The decay column gives Compass's per-step weight decay as a multiple of NorMuon's; NorMuon keeps its own decay throughout. Losses are final validation losses (nats), averaged over the listed number of seeds; $\Delta$ is Compass minus NorMuon on the same seeds with its 95\% confidence interval, and wins counts the seeds on which Compass has the lower loss.}
\label{tab:long-scale}
\small
\setlength{\tabcolsep}{4pt}
\begin{tabular}{@{}rcrrrrr@{}}
\toprule
Tokens & Decay & Seeds & NorMuon & Compass & $\Delta$ [95\% CI] & Wins \\
\midrule
\multicolumn{7}{@{}l}{\textit{Compass (N/N)}}\\
65.5M & $1\times$ & 16 & $4.30562$ & $4.29974$ & $-0.00587\;[-0.01024,-0.00151]$ & 13/16 \\
65.5M & $2\times$ & 16 & $4.30562$ & $4.28556$ & $-0.02005\;[-0.02384,-0.01627]$ & 16/16 \\
\midrule
\multicolumn{7}{@{}l}{\textit{Compass on the NorMuon update}}\\
65.5M & $1\times$ & 16 & $4.30562$ & $4.30110$ & $-0.00452\;[-0.00727,-0.00176]$ & 14/16 \\
262M & $1\times$ & 8 & $3.70793$ & $3.70774$ & $-0.00019\;[-0.00412,0.00375]$ & 3/8 \\
1.05B & $1\times$ & 4 & $3.53628$ & $3.52745$ & $-0.00882\;[-0.03393,0.01628]$ & 2/4 \\
\bottomrule
\end{tabular}
\vspace{-4pt}
\end{table}

Table~\ref{tab:long-scale} extends the comparison with NorMuon to 16 seeds at 65.5M tokens and to budgets of 262M and 1.05B tokens, with each run's cosine schedule set to its full length. These extensions evaluate how the selected learning rates and weight decay transfer to longer training budgets. Compass on the NorMuon update uses NorMuon's learning rate and weight decay, so the two methods differ only in the expert factors, and we trained both at every budget on the same seeds. At 65.5M tokens, Compass has the lower loss, and the gap grows during training (Figure~\ref{fig:trajectory} in Appendix~\ref{sec:new-details}). At 262M and 1.05B tokens, the two methods are close, also on 16.8M held-out tokens and on WikiText-2 (Appendix~\ref{sec:long-extra}). Compass (N/N) has a lower mean loss than NorMuon at both decays, and with twice the decay it is below NorMuon at every evaluation during training; part of this gap comes from the decay itself, because NorMuon with twice its decay is also 0.01254 nats below NorMuon at its default decay on the same seeds. Under a second decay rule chosen before training, which holds the total decay $\sum_t\eta_t\omega$ constant across budgets, Compass (N/N) has the lowest mean loss of the four methods at 262M tokens over nine seeds (Appendix~\ref{sec:long-extra}). Added to NorMuon, the factors therefore lower its loss at 65.5M tokens and give a similar loss at the longer budgets.

\subsection{Language-blocked data order}
\label{sec:order-results}

FineWeb-Edu shuffles its training data in random order. To test the setting that motivates Compass, in which the data seen by the experts change during training as in curriculum or continual learning, we also train on a multilingual corpus in two data orders. The FineWeb2 model has 64 experts per layer (461.2M total parameters) and trains for 4,000 steps with four sequences of length 1,024 per step (16.4M tokens). Eight languages contribute equal numbers of tokens. In the blocked order, each seed cycles five times through eight blocks of 100 single-language steps. The shuffled order uses the same 4,000 batches in random order, and both orders share the initialization and the held-out data. We tune learning rates separately for each method and order. Compass uses the N/H configuration, and we also add its factors to the NorMuon update. Test losses use 65,536 held-out tokens per language.

\begin{table}[t]
\centering
\caption{Held-out FineWeb2 test loss (nats) after 16.4M training tokens, mean $\pm$ standard deviation over eight seeds. The same batches are shuffled or arranged in 100-step language blocks. Compass uses the N/H configuration, and the lower block applies the Compass factors to the expert matrices of the NorMuon update, with the learning-rate grid and weight decay of NorMuon. With shuffled batches all methods lie within 0.005 nats of each other, and the best mean under blocked order is in bold.}
\label{tab:order-test}
\small
\setlength{\tabcolsep}{4pt}
\begin{tabular}{@{}lrr@{}}
\toprule
Method & Shuffled & Blocked \\
\midrule
Compass (N/H) & $4.40861\pm0.00188$ & $5.46396\pm0.04741$ \\
Nesterov Muon & $4.40840\pm0.00249$ & $5.48667\pm0.04560$ \\
Heavy-ball Muon & $4.40923\pm0.00207$ & $5.47008\pm0.05157$ \\
Expert-only Nesterov & $4.40961\pm0.00213$ & $5.47661\pm0.04251$ \\
Constant expert scale & $4.40866\pm0.00380$ & $5.46575\pm0.04341$ \\
\midrule
NorMuon & $4.40463\pm0.00202$ & $5.46953\pm0.04158$ \\
Compass on the NorMuon update & $4.40569\pm0.00362$ & $\mathbf{5.44608}\pm0.04932$ \\
\bottomrule
\end{tabular}
\vspace{-4pt}
\end{table}

Under blocked order, Compass has a lower mean test loss than Nesterov Muon, heavy-ball Muon, NorMuon, expert-only Nesterov, and the constant expert scale (Table~\ref{tab:order-test}). Against Nesterov Muon, it is lower on seven of eight seeds (paired permutation test $p=0.016$) and in all eight languages (Figure~\ref{fig:order-effect}b). We repeated this comparison on eight new seeds with every setting fixed before training and found a similar gap (Table~\ref{tab:order-replication}); over the 16 seeds, Compass is lower on 14 ($p=0.0012$). With shuffled batches, the two methods reach nearly the same loss, and the gap under blocked order is larger than under shuffled order ($p=0.016$). The effect of Compass therefore depends here on the data order. The constant expert scale also comes close to Compass under blocked order. In the short runs at width 1,024, Compass improves on it on all eight seeds.

Added to the NorMuon update, the Compass factors give the lowest test loss under blocked order (Table~\ref{tab:order-test}). This variant is below NorMuon on six of eight seeds and in all eight languages. A replication on eight new seeds gives a similar gap, and over all 16 seeds it is lower on 13 ($p=0.0005$). With shuffled batches, it matches NorMuon (Table~\ref{tab:order-normuon}). The two methods share the update, weight decay, learning-rate grid, and tuning seeds, and each uses its own selected learning rate.

\section{Analysis of training}
\label{sec:analysis}

We now examine expert scales, load balance and routing, response to language switches, and training stability with Compass. Figure~\ref{fig:analysis} uses statistics recorded during the 65.5M-token runs with twice NorMuon's per-step decay.

\begin{figure}[t]
\centering
\includegraphics[width=\linewidth]{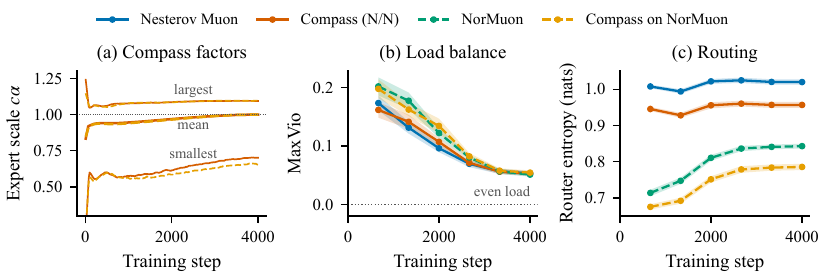}
\caption{\textbf{Training analysis at 65.5M tokens.} All methods use twice NorMuon's per-step weight decay (33 seeds for Compass (N/N), 32 for Nesterov Muon, and 30 each for NorMuon and Compass on the NorMuon update). (a) Smallest, mean, and largest expert scale $c\alpha$ over the 96 expert matrices at each step, averaged over seeds and smoothed over 100 steps. (b) MaxVio; zero means an even load. (c) Mean entropy of the router softmax over all eight experts, computed without the router bias. Panels (b) and (c) use 32,768 fixed training tokens and average over the MoE layers; bands show 95\% confidence intervals.}
\label{fig:analysis}
\end{figure}

\paragraph{Compass factors.} The expert scale $c\alpha$ differs across the 96 expert matrices throughout training (Figure~\ref{fig:analysis}a). After the first 100 steps, the smallest factor stays between about 0.55 and 0.7 and the largest between 1.05 and 1.10. The mean rises from about 0.93 to one as the radius approaches one, and late in training some experts still take steps about 30\% shorter than the average.

\paragraph{Load balance and routing.} Moving step size between experts could change the routing. All four optimizers keep the expert load balanced (Figure~\ref{fig:analysis}b): for every method, MaxVio \citep{wang2024lossfree}, the relative excess load of the most used expert averaged over the MoE layers, falls from about 0.2 at the first evaluation to about 0.05 at the end of training. Every expert receives tokens in all 165 runs of these methods at 65.5M tokens, at NorMuon's per-step decay and at twice that decay. The router distribution over experts is more concentrated with Compass (Figure~\ref{fig:analysis}c). Its mean entropy, which Excitation also reports \citep{shaier2026excitation}, is lower with Compass than with the same update alone across seeds. It is 0.957 against 1.020 nats with the Muon update (32 seeds, $t=-16.0$, $p<10^{-15}$) and 0.786 against 0.843 nats with the NorMuon update (30 seeds, $t=-12.4$, $p<10^{-12}$). The gap persists at the longer budgets of Table~\ref{tab:long-scale}, where the losses of the two methods are close. Compass on the NorMuon update has the lower router entropy on all eight seeds at 262M tokens ($t=-7.2$, $p=0.0002$) and on three of the four seeds at 1.05B tokens. Both runs update the router the same way, so the difference arises from the expert updates; at the same load balance, tokens are assigned more clearly to particular experts.

\paragraph{Language switches.} The factors also respond when the data seen by the experts change. In the multilingual runs of Section~\ref{sec:order-results}, we recorded the spread of $c\alpha$ across the expert matrices of Compass at NorMuon update 25, 50, 75, and 100 steps into each block. Twenty-five steps after a language switch, the standard deviation of $c\alpha$ is 0.145 against 0.136 later in the block, and its minimum is 0.26 against 0.30, on all eight seeds ($t=12.3$ and $t=-4.5$). With shuffled batches, the spread at the same steps stays within 0.003 of its later value. Compass therefore shortens the steps of some experts most strongly early in each block, and it lowers the test loss most under blocked order.

\paragraph{Stability.} Moving step size between experts could also make training less stable. None of these 165 runs diverged or produced a non-finite loss. Following \citet{olmo2025olmo2}, we count a loss spike when the training loss at a step is at least seven standard deviations away from the mean of the previous 1,000 steps. No run has such a spike, even with a 100-step window.

\section{Conclusion}

\Compass{} sets each expert matrix step based on the agreement between its orthogonalized update and its current gradient, while keeping the direction computed by Muon. In pretraining on FineWeb-Edu, Compass with Nesterov momentum on all matrices performs as well as or better than Muon, NorMuon, and the other optimizers we tested, with weight decay matched to NorMuon in the longer runs. Added to NorMuon, Compass lowers its loss at 65.5M tokens while yielding similar loss at 262M and 1.05B tokens. In these runs, the expert load stays balanced, and the router assigns tokens to experts more decisively than with the same update alone. Compass is most useful when the data seen by the experts change, as with language blocks, reducing steps more strongly early in the block. These results suggest that expert step sizes matter more when the data of each expert change during training, as in continual pretraining or with data mixtures that vary over time.
\label{sec:mainend}

\FloatBarrier\bigskip

\section*{Ethics statement}
This work uses publicly released language models and datasets, and their usage conditions apply to the trained checkpoints and derived data. We keep training, model selection, and held-out evaluation data separate. Models trained on these corpora can reproduce biases and harmful content in the data, so any deployment needs a separate evaluation.

\section*{AI use statement}
Generative AI tools assisted with implementation, analysis, and writing; the authors verified all code, proofs, and results and take full responsibility for the content.

\bibliography{references,theory_references}
\bibliographystyle{iclr2027_conference}

\clearpage
\appendix
\section{Experimental details for the pretraining runs}
\label{sec:new-details}

This appendix gives the architectures and training settings, additional runs at 262M and 1.05B tokens, the data splits, and the learning-rate selection and baselines.

\subsection{Architectures and training}
All pretraining models are decoder-only Transformers with six pre-norm blocks, learned positional embeddings, full causal attention, and a final LayerNorm. Table~\ref{tab:arch} gives their sizes. Each expert has two linear projections with a GELU between them; the embedding and output layers are untied, and we use the GPT-2 tokenizer (vocabulary size 50,257).

\begin{table}[ht]
\centering
\caption{Model sizes. In the FineWeb-Edu models, the expert hidden width is twice the model width. The multilingual model has about 71.08M active parameters per token.}
\label{tab:arch}
\small
\setlength{\tabcolsep}{4pt}
\begin{tabular}{@{}lrrr@{}}
\toprule
& \multicolumn{2}{c}{FineWeb-Edu} & \\
\cmidrule(lr){2-3}
Setting & 8.2M tokens & 65.5M to 1.05B tokens & Multilingual \\
\midrule
Model width & 1,024 / 2,048 & 1,024 & 512 \\
Attention heads & 16 & 16 & 8 \\
Experts per layer & 8 & 8 & 64 \\
Experts per token & 2 & 2 & 2 \\
Expert hidden width & 2,048 / 4,096 & 2,048 & 1,024 \\
Context length & 256 & 1,024 & 1,024 \\
Total parameters & 329,781,248 / 1,112,547,328 & 330,567,680 & 461,154,304 \\
\bottomrule
\end{tabular}
\end{table}

The router takes a softmax over all experts, keeps the two with the highest score plus bias, and renormalizes their probabilities. After each step, we update the bias of expert $e$ as $\xi_e\gets\xi_e+0.02\operatorname{sign}(1/E-\ell_e)$, where $\ell_e$ is its share of the layer's routed tokens; this rule is the same for all optimizers, and evaluation uses the final biases. We clip gradients to a global norm of one and run all computation in FP32, without automatic mixed precision or TF32.

\begin{algorithm}[H]
\caption{One training step of Compass}
\label{alg:compass}
\begin{algorithmic}[1]
\Require Clipped gradients, weights, momentum buffers and learning rates at step $t$
\State Apply the AdamW update to embeddings, output head, biases, and normalization parameters
\For{each attention or router matrix with a gradient}
\State $M\gets0.95M+G$, $O\gets\NS(M)$
\State Apply weight decay and add $-\eta_t sO$ to the weights
\EndFor
\For{each expert family}
\State Include every matrix that has a gradient, including zero gradients
\For{each included matrix $e$}
\State $M_e\gets0.95M_e+G_e$, $O_e\gets\NS(G_e+0.95M_e)$
\State Compute the alignment score $a_e$ from $O_e$ and $G_e$
\EndFor
\State Compute $c_e$ with Equation~\eqref{eq:expertgate}
\For{each included matrix $e$}
\State Compute the row factors and $\alpha_e$ with Equations~\eqref{eq:rowgate} and~\eqref{eq:radius}
\State Apply weight decay and add $-(\eta_t c_e)s(\alpha_e O_e)$ to $W_e$
\EndFor
\EndFor
\end{algorithmic}
\end{algorithm}

For Compass (N/N), line 3 of Algorithm~\ref{alg:compass} uses $\NS(G+0.95M)$ in place of $\NS(M)$, and the rest of the algorithm is unchanged. Its data, architecture, schedule, and seeds match the other runs at each budget, and Table~\ref{tab:current-rates} gives its learning-rate sweeps.

For the runs with NorMuon's per-step decay, we set $\omega=1.6\times10^{-4}/\eta_{\max}$, so that $\eta_t\omega=1.6\times10^{-4}\lambda_t$ at every step. Table~\ref{tab:current-rates} gives the learning-rate grids and the selected rates of these runs; among the Compass settings, twice this decay at 0.00175 has the lowest mean tuning loss. Compass on the NorMuon update applies the Compass factors to the expert matrices of the NorMuon update from the reference implementation, with NorMuon's own weight decay, and measures its cosines on the NorMuon update itself. As in NorMuon, an expert without a gradient in a step receives a zero gradient and still takes an update from its momentum, without weight decay or Compass factors.

Each FineWeb-Edu run uses one NVIDIA RTX PRO 6000 Blackwell GPU (PyTorch 2.11, CUDA 12.8), except the 262M- and 1.05B-token pairs of Table~\ref{tab:long-scale}, which use one GH200 GPU per run (PyTorch 2.12, CUDA 13.2) with the same code, settings, and tokenized data. The two methods of a seed always run on the same GPU type. Eleven NorMuon runs at 262M and 1.05B tokens, repeated on different nodes, reproduced their final validation losses bit for bit. The multilingual runs use GH200 GPUs and run several independent runs on one GPU, which yields the same parameters, optimizer states, and losses as running alone. Table~\ref{tab:timing} gives the wall-clock cost of the Compass factors.

\begin{table}[ht]
\centering
\caption{Wall-clock cost of the Compass factors. Each pair uses the same seed on the same machine, and the time covers the whole run, including validation. The last column gives the extra time of Compass relative to the method without the factors.}
\label{tab:timing}
\small
\setlength{\tabcolsep}{4pt}
\begin{tabular}{@{}lllrr@{}}
\toprule
Tokens & GPU & Comparison & Pairs & Extra time \\
\midrule
65.5M & RTX PRO 6000 & Compass (N/N) and Nesterov Muon, twice the decay & 4 & 4.3\% \\
65.5M & RTX PRO 6000 & Compass on the NorMuon update and NorMuon & 5 & 3.8\% \\
262M & GH200 & Compass on the NorMuon update and NorMuon & 8 & 5.2\% \\
1.05B & GH200 & Compass on the NorMuon update and NorMuon & 4 & 6.6\% \\
\bottomrule
\end{tabular}
\end{table}

\subsection{Additional runs at 262M and 1.05B tokens}
\label{sec:long-extra}

\paragraph{More seeds.} We also trained long-budget pairs on the RTX PRO 6000 GPUs: Compass on the NorMuon update and NorMuon for seeds 250 to 252 at 262M tokens and seed 250 at 1.05B tokens, and Compass (N/N) and NorMuon at 262M tokens with three seeds at NorMuon's per-step decay and two at twice that decay, where the two methods differ by 0.0037 and 0.0061 nats. Pooled with the pairs of Table~\ref{tab:long-scale}, Compass on the NorMuon update and NorMuon are within 0.0003 nats at 262M tokens (11 seeds, 95\% CI $[-0.0026,0.0033]$, lower on 5), and Compass on the NorMuon update is 0.0058 nats below NorMuon at 1.05B tokens (5 seeds, 95\% CI $[-0.0247,0.0130]$, lower on 2). The additional seeds therefore leave the conclusion of Table~\ref{tab:long-scale} unchanged: at these budgets, the two methods are close.

\paragraph{Larger evaluation sets.} To reduce sampling noise, we re-evaluated the final models of the long-budget pairs on the first 16.8M tokens of the validation split, 128 times the usual 131,072 tokens, in non-overlapping windows of 1,024 tokens. On the seeds of Table~\ref{tab:long-scale}, Compass on the NorMuon update is 0.0009 nats below NorMuon at 262M tokens (95\% CI $[-0.0039,0.0022]$) and 0.0073 nats below it at 1.05B tokens (95\% CI $[-0.0309,0.0163]$). On WikiText-2, the two methods differ by 0.004 and 0.011 nats at the two budgets. These differences are small compared with the spread of 0.17 to 0.18 nats across NorMuon's seeds, and adding the RTX PRO 6000 pairs leads to the same conclusions. The larger evaluation sets therefore agree with Table~\ref{tab:long-scale}.

\paragraph{A decay rule chosen before training.} We also trained the Muon- and NorMuon-based pairs at both budgets under a decay rule that we chose before training. The rule keeps the total decay $\sum_t\eta_t\omega$ fixed across budgets; per step, this is twice NorMuon's decay at 65.5M tokens, half at 262M tokens, and one eighth at 1.05B tokens. The rule raises the final loss of NorMuon and of Compass on the NorMuon update by 0.0072 and 0.0075 nats at 262M tokens (on all eight GH200 seeds) and by 0.039 and 0.050 nats at 1.05B tokens (four and three seeds), so Table~\ref{tab:long-scale} uses NorMuon's decay. Under the rule, Compass on the NorMuon update is 0.0004 nats below NorMuon at 262M tokens (nine seeds, 95\% CI $[-0.0024,0.0016]$), and the two differ by less than 0.0001 nats at 1.05B tokens (five seeds, 95\% CI $[-0.0066,0.0067]$).

At 262M tokens under the rule, the mean losses over nine seeds are 3.7088 for Compass (N/N), 3.7102 for Nesterov Muon, 3.7150 for Compass on the NorMuon update, and 3.7154 for NorMuon. Compass (N/N) is below NorMuon on eight of the nine seeds (95\% CI $[-0.0119,-0.0014]$, paired permutation test $p=0.012$) and below Nesterov Muon on seven (95\% CI $[-0.0059,0.0031]$). At 1.05B tokens, Compass (N/N) and Nesterov Muon differ by 0.0050 and 0.0051 nats on the two seeds trained (252 and 253). In these comparisons, the pair with seed 250 ran on the RTX PRO 6000 GPUs, and we exclude one Compass run on the NorMuon update (262M tokens, seed 251) that lacks a matching NorMuon run. Under this rule, Compass (N/N) has the lowest mean loss at 262M tokens, and Compass on the NorMuon update stays close to NorMuon at both budgets.

\subsection{Data}
We use the \texttt{sample-10BT} subset of FineWeb-Edu. Each document goes to the training split (95\%) or the validation split (5\%), and we separate documents by the end-of-text token. The split for runs up to 65.5M tokens has 57M training and 3M validation tokens, and the larger splits of Section~\ref{sec:experiments} are built in the same way from the same document stream. Training batches are random windows drawn with a seed-dependent generator, and the target at each position is the next token. Validation uses eight fixed batches drawn with seed 12345, with 32,768 tokens for the short runs and 131,072 for the long runs, all of the same shape, so their average equals the token-level mean loss.

The multilingual experiment uses French, German, Spanish, Italian, Portuguese, Polish, Dutch, and Indonesian from FineWeb2. We grouped documents by normalized URL, removed exact duplicates, and split the groups into training, validation, and test sets in proportions of 90\%, 5\%, and 5\%. In the blocked order, a seed-specific permutation of the languages sets their order in the five cycles of eight 100-step blocks, and the shuffled order applies a second permutation to the same batches. We validate on 8,192 tokens per language every 100 steps. We fixed splits and batch orders before any training run.

\subsection{Learning-rate selection and baselines}
The selection rule in Section~\ref{sec:experiments} applies to both widths, to 65.5M tokens, and to each data order in the multilingual experiment; ties go to the lower rate. The tuning seeds are 3500 and 3501 at width 1,024 and 3600 and 3601 at width 2,048, and the evaluation seeds are 250 to 257 and 260 to 267 at the two widths. At 65.5M tokens, we added eight evaluation seeds (258 to 265), and the 262M- and 1.05B-token pairs of Table~\ref{tab:long-scale} use seeds 267 to 274 and 251 to 254. The constant expert scale is the mean of $c\alpha$ over steps and expert matrices in a separate Compass run with the same settings. We computed it separately for each width, for the 65.5M-token setting, and for each data order; for multilingual training it is 0.940700 (shuffled) and 0.948149 (blocked).

Nesterov Muon and heavy-ball Muon share the shape factor $0.2\sqrt{\max(m,n)}$, the parameter groups, the AdamW learning rate, gradient clipping, schedule, and weight decay of Compass. For NorMuon, we ran the single-GPU reference implementation of \citet{li2025normuon} with its Newton-Schulz iteration in FP32; its own update scaling puts its learning rates in Table~\ref{tab:current-rates} on a different scale. Excitation follows the published utilization-based rule on heavy-ball Muon with exponent one, and its multiplier scales the full expert step, weight decay included. For Magma, we used a public AdamW-based implementation on the hidden matrices (masking probability 0.5, temperature 2, alignment averaging coefficient 0.9). AdamW uses the settings of Appendix~\ref{sec:numerics}, both as a baseline and for the embedding, output-head, bias, and normalization parameters.

\small
\begin{longtable}{@{}>{\raggedright\arraybackslash}p{.17\linewidth}>{\raggedright\arraybackslash}p{.25\linewidth}>{\raggedright\arraybackslash}p{.38\linewidth}r@{}}
\caption{Learning-rate sweeps. Each entry gives a learning rate and the mean final validation loss (nats) over the two tuning seeds; the last column gives the selected rate.}\label{tab:current-rates}\\
\toprule Setting & Method & Learning rate: mean tuning loss & Selected \\ \midrule
\endfirsthead
\toprule Setting & Method & Learning rate: mean tuning loss & Selected \\ \midrule
\endhead
8.2M, 1,024 & Nesterov Muon & 0.0005625: 5.58323, 0.001125: 5.46007, 0.00225: 5.61247 & 0.001125 \\
8.2M, 2,048 & Nesterov Muon & 0.00034375: 5.43610, 0.0006875: 5.32134, 0.001375: 5.56518 & 0.0006875 \\
8.2M, 1,024 & Compass (N/H) & 0.000625: 5.55381, 0.00125: 5.45141, 0.0025: 5.85329 & 0.00125 \\
8.2M, 2,048 & Compass (N/H) & 0.000375: 5.41242, 0.00075: 5.31861, 0.0015: 5.66832 & 0.00075 \\
8.2M, 1,024 & NorMuon & 0.004: 5.58115, 0.008: 5.46068, 0.016: 5.64487 & 0.008 \\
8.2M, 2,048 & NorMuon & 0.003125: 5.46190, 0.00625: 5.32107, 0.0125: 5.46869 & 0.00625 \\
8.2M, 1,024 & Excitation & 0.0005625: 5.57691, 0.001125: 5.45991, 0.00225: 5.70183 & 0.001125 \\
8.2M, 2,048 & Excitation & 0.00034375: 5.43100, 0.0006875: 5.31719, 0.001375: 5.62580 & 0.0006875 \\
8.2M, 1,024 & Magma & 0.0006: 6.07928, 0.0012: 5.87716, 0.0024: 6.12660 & 0.0012 \\
8.2M, 2,048 & Magma & 0.0004: 5.90850, 0.0008: 5.87356, 0.0016: 6.31308 & 0.0008 \\
8.2M, 1,024 & AdamW & 0.0003: 6.02889, 0.0006: 6.10369, 0.0012: 6.59153 & 0.0003 \\
8.2M, 2,048 & AdamW & 0.0002: 5.98742, 0.0004: 6.43759, 0.0008: 6.85597 & 0.0002 \\
8.2M, 1,024 & Heavy-ball Muon & 0.0005625: 5.57384, 0.001125: 5.46451, 0.00225: 5.68255 & 0.001125 \\
8.2M, 2,048 & Heavy-ball Muon & 0.00034375: 5.43027, 0.0006875: 5.32990, 0.001375: 5.53181 & 0.0006875 \\
65.5M, 1,024 & Compass (N/H) & 0.000625: 4.48751, 0.00125: 4.34896, 0.0025: 4.34744 & 0.0025 \\
65.5M, 1,024 & Nesterov Muon & 0.0005625: 4.52486, 0.001125: 4.35155, 0.00225: 4.31299 & 0.00225 \\
8.2M, 1,024 & Expert-only Nesterov & 0.000625: 5.55031, 0.00125: 5.46093, 0.0025: 5.76266 & 0.00125 \\
8.2M, 2,048 & Expert-only Nesterov & 0.000375: 5.40946, 0.00075: 5.31764, 0.0015: 5.77167 & 0.00075 \\
8.2M, 1,024 & Constant expert scale & 0.000625: 5.54809, 0.00125: 5.46183, 0.0025: 5.77364 & 0.00125 \\
8.2M, 2,048 & Constant expert scale & 0.000375: 5.41191, 0.00075: 5.32046, 0.0015: 5.60901 & 0.00075 \\
8.2M, 1,024 & Family factor only & 0.000625: 5.55013, 0.00125: 5.46234, 0.0025: 5.78520 & 0.00125 \\
8.2M, 2,048 & Family factor only & 0.000375: 5.40639, 0.00075: 5.31673, 0.0015: 5.67988 & 0.00075 \\
8.2M, 1,024 & Radius only & 0.000625: 5.55394, 0.00125: 5.45810, 0.0025: 5.75885 & 0.00125 \\
8.2M, 2,048 & Radius only & 0.000375: 5.41524, 0.00075: 5.31524, 0.0015: 5.69521 & 0.00075 \\
65.5M, 1,024 & Expert-only Nesterov & 0.000625: 4.49086, 0.00125: 4.34898, 0.0025: 4.34528 & 0.0025 \\
65.5M, 1,024 & Constant expert scale & 0.000625: 4.49121, 0.00125: 4.34978, 0.0025: 4.34448 & 0.0025 \\
65.5M, 1,024 & NorMuon & 0.004: 4.52649, 0.008: 4.35955, 0.016: 4.30735 & 0.016 \\
65.5M, 1,024 & Compass (N/N) & 0.000625: 4.49313, 0.00125: 4.33468, 0.0025: 4.32614 & 0.0025 \\
8.2M, 1,024 & Compass (N/N) & 0.000625: 5.55654, 0.00125: 5.45029, 0.0025: 5.72406 & 0.00125 \\
8.2M, 2,048 & Compass (N/N) & 0.000375: 5.41704, 0.00075: 5.31271, 0.0015: 5.60675 & 0.00075 \\
65.5M, NorMuon decay & Compass (N/N) & 0.00125: 4.31849, 0.00175: 4.29882, 0.0025: 4.30147 & 0.00175 \\
65.5M, NorMuon decay & Compass (N/N), $2\times$ decay & 0.00125: 4.30521, 0.00175: 4.28936 & 0.00175 \\
65.5M, NorMuon decay & Compass (N/N), $4\times$ decay & 0.00125: 4.29843 & \\
65.5M, NorMuon decay & Nesterov Muon & 0.00125: 4.31992, 0.00175: 4.30170, 0.0025: 4.30437 & 0.00175 \\
65.5M, NorMuon decay & Nesterov Muon, $2\times$ decay & 0.00125: 4.31241, 0.00175: 4.28181 & 0.00175 \\
65.5M, NorMuon decay & NorMuon, $2\times$ decay & 0.008: 4.33258, 0.016: 4.28653 & 0.016 \\
65.5M, NorMuon decay & Compass on the NorMuon update & 0.016: 4.30626, 0.022: 4.33721 & 0.016 \\
\bottomrule
\end{longtable}
\normalsize

For multilingual training, Compass, Nesterov Muon, heavy-ball Muon, expert-only Nesterov, and the constant expert scale are tuned for each order over $\{0.000625,0.00125,0.0025\}$ on seeds 3500 and 3501, and the evaluation runs use seeds 250 to 257. All but one of the ten selections are 0.00125. The exception is Nesterov Muon under blocked order, which picks 0.000625 by a margin of $7\times10^{-6}$ nats in mean tuning loss (5.46943). NorMuon and Compass on the NorMuon update were tuned on the same seeds over $\{0.00283,0.00566,0.0113\}$, a grid on NorMuon's own scale. NorMuon chose 0.00283 (blocked) and 0.00566 (shuffled), and Compass on the NorMuon update chose 0.00566 for both orders. The NorMuon tuning runs used a second GH200 system (PyTorch 2.7, CUDA 12.8); all multilingual evaluation runs used PyTorch 2.12 and CUDA 13.2.

Figure~\ref{fig:trajectory} shows the validation loss during training at 65.5M tokens.

\begin{figure}[ht]
\centering
\includegraphics[width=\linewidth]{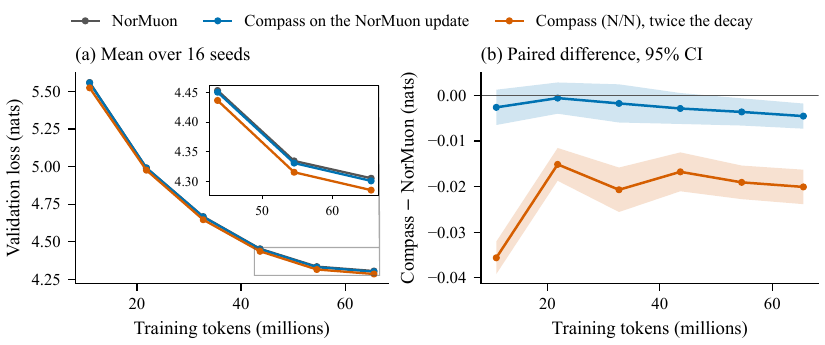}
\caption{Validation loss during training at 65.5M tokens (width 1,024, context length 1,024; seeds 250 to 265, evaluated every 666 updates). (a) Mean over the 16 seeds; the inset enlarges the last three evaluations. (b) Compass minus NorMuon on the same seed, with pointwise 95\% paired-$t$ confidence bands. NorMuon and Compass on the NorMuon update use NorMuon's weight decay; Compass (N/N) uses twice NorMuon's per-step decay, as in Table~\ref{tab:long-scale}.}
\label{fig:trajectory}
\end{figure}

\section{Paired comparisons and statistical tests}
\label{sec:new-statistics}
For each seed, the paired difference is the loss of Compass minus the loss of the baseline. We report means and sample standard deviations over seeds, and 95\% confidence intervals and two-sided $p$-values from paired $t$-tests. For the main comparisons, we also give exact two-sided $p$-values from paired permutation (sign-flip) tests, which enumerate all $2^n$ sign assignments. Seeds are the independent replicates; results for different languages come from the same trained models and are therefore correlated.

Table~\ref{tab:current-attribution} gives the paired $t$ statistics and $p$-values of the ablations at 8.2M tokens, Table~\ref{tab:current-paired} lists every paired comparison, and Table~\ref{tab:wikitext-long} gives the WikiText-2 test loss of the 65.5M-token models. Table~\ref{tab:extra-paired} adds the comparisons on all evaluation seeds at twice the decay, on the usual validation tokens and on the full validation split, and the paired WikiText-2 differences.

\begin{figure}[ht]
\centering
\includegraphics[width=\linewidth]{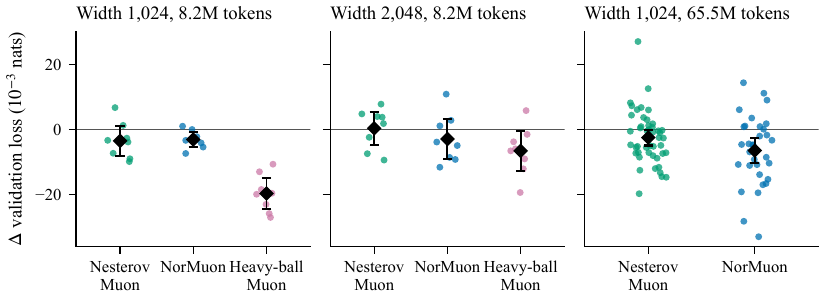}
\caption{\textbf{Paired differences in final validation loss between Compass and each baseline.} Each point is one seed; diamonds and bars show the mean and 95\% confidence interval. Negative values favor Compass. At 8.2M tokens, Compass uses Nesterov momentum on all matrices (N/N), with eight seeds per comparison. At 65.5M tokens, Compass (N/N), Nesterov Muon, and NorMuon all use twice NorMuon's per-step weight decay, with 46 seeds for Nesterov Muon and 32 for NorMuon (Section~\ref{sec:long-context}).}
\label{fig:paired-current}
\end{figure}

\begin{table}[t]
\centering
\caption{Ablations at 8.2M tokens. Compass (N/H) minus each ablation in final validation loss (nats); negative values favor Compass. Intervals are 95\% confidence intervals from paired $t$-tests over eight seeds, and $t$ and $p$ are the paired $t$ statistic and its two-sided $p$-value.}
\label{tab:current-attribution}
\small
\setlength{\tabcolsep}{4pt}
\begin{tabular}{@{}rlrrr@{}}
\toprule
Width & Ablation & Difference [95\% CI] & $t$ & $p$ \\
\midrule
1,024 & Expert-only Nesterov & $-0.00603\;[-0.01102,-0.00104]$ & $-2.86$ & 0.024 \\
2,048 & Expert-only Nesterov & $-0.00576\;[-0.01212,+0.00059]$ & $-2.14$ & 0.069 \\
1,024 & Constant expert scale & $-0.00749\;[-0.00934,-0.00564]$ & $-9.56$ & 0.00003 \\
2,048 & Constant expert scale & $-0.00298\;[-0.00872,+0.00276]$ & $-1.23$ & 0.26 \\
\bottomrule
\end{tabular}
\end{table}

\FloatBarrier\bigskip
\footnotesize
\setlength{\tabcolsep}{3pt}
\begin{longtable}{@{}lllrrrr@{}}
\caption{Paired comparisons of final validation loss: Compass minus the comparison method (nats), 95\% confidence interval, seeds favoring Compass, and two-sided paired $t$-test $p$-value. Settings give the token budget and the model width; $1\times$ and $2\times$ denote NorMuon's per-step weight decay and twice that value.}\label{tab:current-paired}\\
\toprule Setting & Compass & Comparison & Difference & 95\% CI & Wins & $p$ \\ \midrule
\endfirsthead
\toprule Setting & Compass & Comparison & Difference & 95\% CI & Wins & $p$ \\ \midrule
\endhead
8.2M, 1,024 & N/N & NorMuon & $-0.00303$ & $[-0.00532,-0.00074]$ & 6/8 & 0.016 \\
8.2M, 2,048 & N/N & NorMuon & $-0.00286$ & $[-0.00913,+0.00340]$ & 5/8 & 0.32 \\
8.2M, 1,024 & N/N & Nesterov Muon & $-0.00344$ & $[-0.00807,+0.00119]$ & 6/8 & 0.12 \\
8.2M, 2,048 & N/N & Nesterov Muon & $+0.00044$ & $[-0.00473,+0.00560]$ & 3/8 & 0.85 \\
8.2M, 1,024 & N/H & Nesterov Muon & $+0.00144$ & $[-0.00415,+0.00704]$ & 5/8 & 0.56 \\
8.2M, 2,048 & N/H & Nesterov Muon & $-0.00089$ & $[-0.00562,+0.00384]$ & 4/8 & 0.67 \\
8.2M, 1,024 & N/H & NorMuon & $+0.00186$ & $[-0.00127,+0.00498]$ & 2/8 & 0.20 \\
8.2M, 2,048 & N/H & NorMuon & $-0.00419$ & $[-0.00860,+0.00022]$ & 6/8 & 0.059 \\
8.2M, 1,024 & N/H & Excitation & $-0.01033$ & $[-0.01425,-0.00641]$ & 8/8 & $4.3\times10^{-4}$ \\
8.2M, 2,048 & N/H & Excitation & $-0.00582$ & $[-0.01004,-0.00159]$ & 8/8 & 0.014 \\
8.2M, 1,024 & N/H & Magma & $-0.40414$ & $[-0.41412,-0.39417]$ & 8/8 & $3.6\times10^{-12}$ \\
8.2M, 2,048 & N/H & Magma & $-0.50486$ & $[-0.53970,-0.47003]$ & 8/8 & $4.7\times10^{-9}$ \\
8.2M, 1,024 & N/H & AdamW & $-0.58555$ & $[-0.59283,-0.57828]$ & 8/8 & $2.9\times10^{-14}$ \\
8.2M, 2,048 & N/H & AdamW & $-0.61430$ & $[-0.63249,-0.59610]$ & 8/8 & $1.3\times10^{-11}$ \\
8.2M, 1,024 & N/H & Heavy-ball Muon & $-0.01479$ & $[-0.01999,-0.00959]$ & 8/8 & $2.7\times10^{-4}$ \\
8.2M, 2,048 & N/H & Heavy-ball Muon & $-0.00784$ & $[-0.01245,-0.00323]$ & 8/8 & 0.0051 \\
\midrule
8.2M, 1,024 & N/H & Expert-only Nesterov & $-0.00603$ & $[-0.01102,-0.00104]$ & 8/8 & 0.024 \\
8.2M, 2,048 & N/H & Expert-only Nesterov & $-0.00576$ & $[-0.01212,+0.00059]$ & 7/8 & 0.069 \\
8.2M, 1,024 & N/H & Constant expert scale & $-0.00749$ & $[-0.00934,-0.00564]$ & 8/8 & $2.9\times10^{-5}$ \\
8.2M, 2,048 & N/H & Constant expert scale & $-0.00298$ & $[-0.00872,+0.00276]$ & 5/8 & 0.26 \\
8.2M, 1,024 & N/H & Family factor only & $-0.00298$ & $[-0.00429,-0.00167]$ & 8/8 & 0.0010 \\
8.2M, 2,048 & N/H & Family factor only & $-0.00151$ & $[-0.00566,+0.00264]$ & 5/8 & 0.42 \\
8.2M, 1,024 & N/H & Radius only & $-0.00131$ & $[-0.00586,+0.00324]$ & 4/8 & 0.52 \\
8.2M, 2,048 & N/H & Radius only & $+0.00021$ & $[-0.00405,+0.00448]$ & 3/8 & 0.91 \\
\midrule
65.5M, 1,024 & N/H & Nesterov Muon & $+0.03141$ & $[+0.02351,+0.03931]$ & 0/8 & $3.2\times10^{-5}$ \\
65.5M, 1,024 & N/H & NorMuon & $+0.04328$ & $[+0.03669,+0.04988]$ & 0/8 & $1.1\times10^{-6}$ \\
65.5M, 1,024 & N/H & Expert-only Nesterov & $-0.00132$ & $[-0.01038,+0.00775]$ & 4/8 & 0.74 \\
65.5M, 1,024 & N/H & Constant expert scale & $+0.00072$ & $[-0.00796,+0.00940]$ & 4/8 & 0.85 \\
65.5M, 1,024 & N/N, $1\times$ & NorMuon & $-0.00587$ & $[-0.01024,-0.00151]$ & 13/16 & 0.012 \\
65.5M, 1,024 & N/N, $1\times$ & Nesterov Muon & $-0.01647$ & $[-0.02271,-0.01023]$ & 8/8 & $4.3\times10^{-4}$ \\
65.5M, 1,024 & N/N, $1\times$ & Nesterov Muon, $1\times$ & $-0.00032$ & $[-0.00809,+0.00745]$ & 2/8 & 0.93 \\
65.5M, 1,024 & N/N, $2\times$ & NorMuon & $-0.02163$ & $[-0.02707,-0.01620]$ & 8/8 & $3.2\times10^{-5}$ \\
65.5M, 1,024 & N/N, $2\times$ & Nesterov Muon, $1\times$ & $-0.01735$ & $[-0.02382,-0.01088]$ & 8/8 & $3.9\times10^{-4}$ \\
65.5M, 1,024 & N/N, $2\times$ & Compass (N/N), $1\times$ & $-0.01703$ & $[-0.02032,-0.01374]$ & 8/8 & $5.6\times10^{-6}$ \\
65.5M, 1,024 & N/N, $2\times$ & Nesterov Muon, $2\times$ & $-0.00269$ & $[-0.00839,+0.00301]$ & 6/8 & 0.30 \\
65.5M, 1,024 & N/N, $2\times$ & NorMuon, $2\times$ & $-0.00790$ & $[-0.01747,+0.00167]$ & 5/8 & 0.092 \\
\bottomrule
\end{longtable}
\normalsize

\begin{table}[ht]
\centering
\caption{WikiText-2 test loss (nats) of the models trained for 65.5M tokens at context length 1,024, mean $\pm$ standard deviation over eight seeds. The evaluation scores all 286,178 test tokens and resets the context at document boundaries. The best mean is in bold.}
\label{tab:wikitext-long}
\small
\begin{tabular}{@{}lcr@{}}
\toprule
Method & Momentum & WikiText-2 test loss \\
\midrule
NorMuon & N/N & $5.8705\pm0.0337$ \\
Nesterov Muon & N/N & $5.8785\pm0.0265$ \\
Nesterov Muon, NorMuon decay & N/N & $5.8438\pm0.0220$ \\
Compass, NorMuon decay & N/N & $5.8381\pm0.0319$ \\
Compass, $2\times$ NorMuon decay & N/N & $\mathbf{5.8031}\pm0.0350$ \\
Compass on the NorMuon update & N/N & $5.8592\pm0.0395$ \\
\bottomrule
\end{tabular}
\end{table}

\begin{table}[ht]
\centering
\caption{Additional paired comparisons at 65.5M tokens. The validation rows use all evaluation seeds at twice NorMuon's per-step decay: seeds 250 to 281 for the comparisons with NorMuon and 250 to 295 for Nesterov Muon. The full-validation rows score the same models on the whole validation split of about 3M tokens, 23 times the usual 131,072, and their paired $t$-test $p$-values are below 0.001, 0.014, and 0.035; for the validation rows they are 0.002, 0.051, and 0.048. The WikiText-2 rows use the models of Table~\ref{tab:wikitext-long}, and their exact two-sided paired permutation $p$-values are 0.023 and 0.031. Differences are Compass minus the comparison method, in nats, with 95\% confidence intervals from paired $t$-tests; wins count the seeds on which Compass has the lower loss. ``NorMuon update'' marks Compass on the NorMuon update.}
\label{tab:extra-paired}
\small
\setlength{\tabcolsep}{4pt}
\begin{tabular}{@{}lllrrr@{}}
\toprule
Loss & Compass & Comparison & Difference & 95\% CI & Wins \\
\midrule
Validation & N/N, $2\times$ & NorMuon, $2\times$ & $-0.00642$ & $[-0.01027,-0.00256]$ & 22/32 \\
Validation & N/N, $2\times$ & Nesterov Muon, $2\times$ & $-0.00242$ & $[-0.00485,0.00001]$ & 30/46 \\
Validation & NorMuon update, $2\times$ & NorMuon, $2\times$ & $-0.00330$ & $[-0.00656,-0.00004]$ & 20/32 \\
Full validation & N/N, $2\times$ & NorMuon, $2\times$ & $-0.00744$ & $[-0.01075,-0.00412]$ & 25/32 \\
Full validation & N/N, $2\times$ & Nesterov Muon, $2\times$ & $-0.00280$ & $[-0.00499,-0.00060]$ & 29/46 \\
Full validation & NorMuon update, $2\times$ & NorMuon, $2\times$ & $-0.00303$ & $[-0.00583,-0.00023]$ & 20/32 \\
WikiText-2 & N/N, $2\times$ & NorMuon & $-0.0674$ & $[-0.1119,-0.0230]$ & 7/8 \\
WikiText-2 & N/N, $2\times$ & Nesterov Muon, $1\times$ & $-0.0408$ & $[-0.0759,-0.0056]$ & 7/8 \\
\bottomrule
\end{tabular}
\end{table}

Table~\ref{tab:order-replication} gives the multilingual replication. Pooled with the first set of runs, it gives 16 paired differences against Nesterov Muon at 0.000625 under blocked order, with mean $-0.01995$ nats (95\% CI $[-0.02992,-0.00998]$), lower on 14 of the 16 seeds (exact permutation $p=0.0012$). A similar blocked-order gap appears on the new seeds.

\begin{table}[ht]
\centering
\caption{\textbf{Multilingual replication, seeds 1250 to 1257.} Held-out test loss of Compass minus each comparison method, in nats, with 95\% confidence intervals from paired $t$-tests and exact two-sided paired permutation $p$-values. The four blocked-order comparisons were chosen before training, $t$ is the paired $t$ statistic, and the last row gives the shuffled order. The replication uses the data, schedule, and learning rates of Section~\ref{sec:order-results} and differs from the first set of runs only in its seeds.}
\label{tab:order-replication}
\footnotesize
\setlength{\tabcolsep}{3pt}
\begin{tabular}{@{}lllrrrrr@{}}
\toprule
Order & Compass & Comparison & Difference & 95\% CI & Wins & $t$ & $p$ \\
\midrule
Blocked & N/H & Nesterov Muon, 0.000625 & $-0.01719$ & $[-0.03641,+0.00204]$ & 7/8 & $-2.11$ & 0.070 \\
Blocked & N/H & Nesterov Muon, 0.00125 & $-0.00886$ & $[-0.02087,+0.00314]$ & 5/8 & $-1.75$ & 0.133 \\
Blocked & N/N & Nesterov Muon, 0.00125 & $-0.00449$ & $[-0.01879,+0.00981]$ & 6/8 & $-0.74$ & 0.500 \\
Blocked & N/H & Constant expert scale & $-0.00223$ & $[-0.01341,+0.00894]$ & 5/8 & $-0.47$ & 0.648 \\
\midrule
Shuffled & N/H & Nesterov Muon, 0.00125 & $+0.00016$ & $[-0.00250,+0.00283]$ & 2/8 & $0.14$ & 0.922 \\
\bottomrule
\end{tabular}
\end{table}

NorMuon and Compass on the NorMuon update were trained after the other methods of Table~\ref{tab:order-test}, on the same data, seeds, and system, and Table~\ref{tab:order-normuon} gives their comparisons. The lower block of Table~\ref{tab:order-normuon} repeats these comparisons on the seeds of Table~\ref{tab:order-replication}, with the learning rates fixed before training. Pooled with the first set of runs, it gives 16 paired differences against NorMuon under blocked order, with mean $-0.02525$ nats (95\% CI $[-0.03654,-0.01396]$), lower on 13 of the 16 seeds (exact permutation $p=0.0005$), and the mean difference favors Compass in all eight languages. The gap to NorMuon under blocked order therefore also holds on new seeds.

\begin{table}[ht]
\centering
\caption{\textbf{Compass on the NorMuon update in the multilingual experiment.} Held-out test loss of Compass on the NorMuon update minus each comparison method, in nats, with 95\% confidence intervals from paired $t$-tests and exact two-sided paired permutation $p$-values. The upper block uses seeds 250 to 257, and the lower block repeats the same comparisons on seeds 1250 to 1257 with the same learning rates. The first four comparisons were chosen before training, and $t$ is the paired $t$ statistic. The last row of each block gives the difference between the blocked and shuffled gaps relative to NorMuon.}
\label{tab:order-normuon}
\footnotesize
\setlength{\tabcolsep}{3pt}
\begin{tabular}{@{}llrrrrr@{}}
\toprule
Order & Comparison & Difference & 95\% CI & Wins & $t$ & $p$ \\
\midrule
\multicolumn{7}{@{}l}{\textit{Seeds 250 to 257}}\\
Blocked & NorMuon & $-0.02346$ & $[-0.04128,-0.00563]$ & 6/8 & $-3.11$ & 0.031 \\
Blocked & Compass (N/H) & $-0.01788$ & $[-0.03087,-0.00489]$ & 7/8 & $-3.25$ & 0.016 \\
Shuffled & NorMuon & $+0.00106$ & $[-0.00079,+0.00292]$ & 3/8 & $1.35$ & 0.219 \\
Shuffled & Nesterov Muon & $-0.00270$ & $[-0.00569,+0.00028]$ & 6/8 & $-2.14$ & 0.086 \\
\addlinespace[2pt]
Blocked minus shuffled & NorMuon & $-0.02452$ & $[-0.04324,-0.00580]$ & 7/8 & $-3.10$ & 0.023 \\
\midrule
\multicolumn{7}{@{}l}{\textit{Seeds 1250 to 1257}}\\
Blocked & NorMuon & $-0.02705$ & $[-0.04574,-0.00835]$ & 7/8 & $-3.42$ & 0.016 \\
Blocked & Compass (N/H) & $-0.02439$ & $[-0.03985,-0.00894]$ & 8/8 & $-3.73$ & 0.008 \\
Shuffled & NorMuon & $+0.00016$ & $[-0.00200,+0.00231]$ & 4/8 & $0.17$ & 0.875 \\
Shuffled & Nesterov Muon & $-0.00286$ & $[-0.00505,-0.00067]$ & 7/8 & $-3.09$ & 0.016 \\
\addlinespace[2pt]
Blocked minus shuffled & NorMuon & $-0.02721$ & $[-0.04521,-0.00920]$ & 7/8 & $-3.57$ & 0.016 \\
\bottomrule
\end{tabular}
\end{table}

\subsection{Language-specific results and training dynamics}
\label{sec:language-results}
\begin{figure}[ht]
\centering
\includegraphics[width=\linewidth]{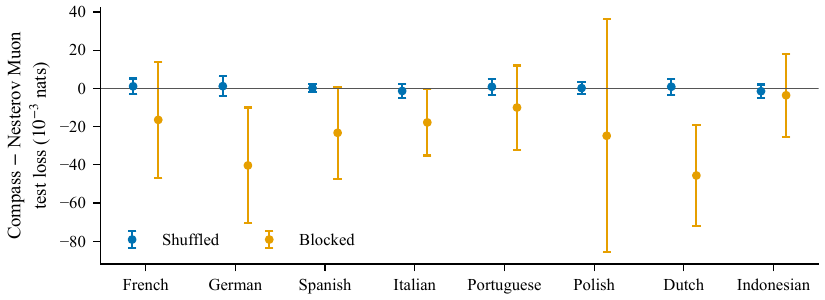}
\caption{\textbf{Test loss by language.} Compass minus Nesterov Muon under shuffled and blocked order, with 95\% confidence intervals over eight seeds. Results for different languages come from the same models. Under blocked order, all eight languages favor Compass, and under shuffled order, every mean difference is at most 0.0015 nats in magnitude.}
\label{fig:language-effect}
\end{figure}

Figure~\ref{fig:language-effect} shows the test loss difference for each language.

The fraction of experts that receive zero tokens is 0.47\% under blocked order and 4.11\% under shuffled order. The mean $c\alpha$ over the four training quarters is (0.899, 0.915, 0.957, 0.993) under shuffling and (0.910, 0.925, 0.963, 0.994) under blocking, and the family factors continue to vary across experts until the end of training.

\FloatBarrier\bigskip
\section{Numerical definition of Compass}
\label{sec:numerics}
\label{app:method-details}

Equations~\eqref{eq:momentum} to~\eqref{eq:update} and the schedule of Equation~\eqref{eq:schedule} define Compass in exact arithmetic. Here, we add the small constants of our code (Table~\ref{tab:method-symbols}), the order of operations (Section~\ref{sec:extended-arithmetic}), and properties of the two factors (Section~\ref{sec:properties}). All runs keep parameters, gradients, and optimizer states in FP32.

\subsection{Symbols and constants}

\begin{table}[ht]
\centering
\small
\caption{Symbols for one layer and one expert projection. A family contains
the corresponding projection from every expert in that layer.}
\label{tab:method-symbols}
\begin{tabular}{p{0.23\linewidth}p{0.69\linewidth}}
\toprule
Symbol & Definition \\
\midrule
$t,T$ & Training step, starting at one, and total number of training steps. \\
$e,E$ & Expert index and number of experts in the family. \\
$r,m,n$ & Row index, number of rows and number of columns of an expert matrix. \\
$W_{e,t}$ & Expert weight matrix before step $t$. \\
$G_{e,t}$ & Expert gradient after global gradient clipping. \\
$M_{e,t}$ & Momentum after its update at step $t$, with $M_{e,0}=0$. \\
$D_{e,t}$ & Input to the Newton-Schulz map, $G_{e,t}+\beta M_{e,t}$. \\
$\mathcal N_5$ & Five iterations of the quintic map in Equation~\eqref{eq:ns}. \\
$s,O_{e,t}$ & Shape factor $0.2\sqrt{\max(m,n)}$ and direction $\mathcal N_5(D_{e,t})$. \\
$z_{e,t},a_{e,t}$ & Cosine of the full matrix and its positive part plus $\epsilon_{\rm score}$. \\
$u_{e,t},v_{e,t},c_{e,t}$ & Normalized expert score, capped score and family factor. \\
$z_{er,t},b_{er,t}$ & Row cosine and its positive part plus $\epsilon_{\rm score}$. \\
$q_{er,t},h_{er,t}$ & Capped row score and row factor after the schedule interpolation. \\
$\alpha_{e,t}$ & Scalar radius computed from the weighted squared row norms. \\
$V_{e,t}$ & Expert update direction, $c_{e,t}\alpha_{e,t}sO_{e,t}$. \\
$\eta_{\max},\eta_t$ & Peak matrix learning rate and its value at step $t$. \\
$\lambda_t$ & Schedule multiplier $\eta_t/\eta_{\max}$. \\
$\beta,\omega$ & Muon momentum coefficient $0.95$ and weight decay $0.01$. \\
$\epsilon_{\rm NS}$ & Normalization constant $10^{-7}$ of the Newton-Schulz map. \\
$\epsilon_{\rm cos}$ & Constant $10^{-12}$ in the cosine denominator. \\
$\epsilon_{\rm score}$ & Constant $10^{-6}$ added to the expert and row scores. \\
$\theta_t,g_t$ & Parameter updated by AdamW and its clipped gradient. \\
$m_t,v_t$ & First and second AdamW moments. \\
$\widehat m_t,\widehat v_t,\tau_t$ & Bias-corrected moments and the number of parameter updates. \\
$\beta_1,\beta_2,\epsilon_A$ & AdamW constants $0.9$, $0.95$ and $10^{-8}$. \\
$\eta_t^{\rm aux},d_t$ & AdamW learning rate $\eta^{\rm aux}_{\max}\lambda_t$ and denominator $\sqrt{\widehat v_t}+\epsilon_A$. \\
\bottomrule
\end{tabular}
\end{table}

For matrices $A$ and $B$ of the same shape, $\langle A,B\rangle_F=\sum_{r,j}A_{rj}B_{rj}$ is the Frobenius inner product and $\lVert A\rVert_F=\sqrt{\langle A,A\rangle_F}$ its norm, and $[x]_+=\max(x,0)$ is the positive part. A family mean gives each expert matrix the same weight, and a row mean gives each row of a matrix the same weight.

\subsection{Order of operations}
\label{sec:extended-arithmetic}

We update the momentum buffer before forming the Nesterov input, so
\begin{equation}
D_{e,t}=G_{e,t}+\beta M_{e,t}
       =(1+\beta)G_{e,t}+\beta^2M_{e,t-1}.
\label{eq:expanded-refresh}
\end{equation}
The Newton-Schulz map divides its input by $\lVert D_{e,t}\rVert_F+\epsilon_{\rm NS}$ and transposes it when $m>n$. It then runs five iterations with the coefficients of \citet{jordan2024muon} and transposes the result back. Each iteration computes
\begin{equation}
A_k=X_kX_k^\top,\qquad
X_{k+1}=3.4445X_k+
\left(-4.7750A_k+2.0315(A_kA_k)\right)X_k.
\label{eq:ns}
\end{equation}

We compute both alignments from the unscaled direction $O_{e,t}=\mathcal N_5(D_{e,t})$ and the current clipped gradient, with $\epsilon_{\rm cos}$ added to each cosine denominator. For the full matrix,
\begin{equation}
z_{e,t}=\frac{\langle O_{e,t},G_{e,t}\rangle_F}
{\lVert O_{e,t}\rVert_F\lVert G_{e,t}\rVert_F+\epsilon_{\rm cos}},
\label{eq:exact-cosine}
\end{equation}
and the row cosine is the same expression for one row, with vector inner products and Euclidean norms. A row score is the positive part of the row cosine plus $\epsilon_{\rm score}$, divided by its mean over the rows and capped at one. Expert scores get the same offset and cap, and the two family normalizations divide by their means, clamped below at $\epsilon_{\rm score}$.

We compute inner products and norms in FP32. For family normalization, we convert values to double precision before computing the cosine and the two family means; the row computations stay in FP32. The row factors interpolate linearly between the capped row scores at the first step and one at the final step, where the learning rate is zero. Writing $h_e$ for the row factors, the radius is
\begin{equation}
 \alpha_e=\begin{cases}
 \sqrt{\displaystyle\frac{\sum_r h_{er}^2\lVert O_{er}\rVert_2^2}
 {\max(\lVert O_e\rVert_F^2,10^{-30})}}, & \lVert O_e\rVert_F^2>0,\\
 1, & \lVert O_e\rVert_F^2=0.
 \end{cases}
\label{eq:extended-radius}
\end{equation}
Weight decay is applied first, and the weight update is
\begin{equation}
 W_e\mathrel{+}=-\bigl(\eta_t c_e\bigr)s\,\alpha_e O_e,\qquad
 s=0.2\sqrt{\max(m,n)}.
\label{eq:extended-addition}
\end{equation}
The radius applies to the direction, and the family and shape factors enter the scalar coefficient. In the N/H configuration, attention and router matrices use heavy-ball momentum. Heavy-ball Muon also uses it for the expert matrices, with both Compass factors set to one.

Each layer has one expert family per projection; row scores are normalized within an expert matrix, and expert scores are normalized across the experts of a family.

\subsection{AdamW update}

Embeddings, the output head, biases, and normalization parameters use AdamW \citep{loshchilov2019adamw}. For such a parameter $\theta_t$ with clipped gradient $g_t$, let $\tau_t$ count the steps up to $t$ on which it had a gradient; the moments start at zero. On a step with a gradient, with $\beta_1=0.9$ and $\beta_2=0.95$,
\begin{align}
m_t&=\beta_1m_{t-1}+(1-\beta_1)g_t,\\
v_t&=\beta_2v_{t-1}+(1-\beta_2)g_t\odot g_t,\\
\widehat m_t&=\frac{m_t}{1-\beta_1^{\tau_t}},\qquad
\widehat v_t=\frac{v_t}{1-\beta_2^{\tau_t}},
\label{eq:aux-moments}
\end{align}
where $\odot$ is elementwise multiplication. The step is $\widehat m_t/d_t$ with $d_t=\sqrt{\widehat v_t}+\epsilon_A$ and $\epsilon_A=10^{-8}$, scaled by the learning rate $\eta_t^{\rm aux}=\eta^{\rm aux}_{\max}\lambda_t$, and the update uses decoupled weight decay 0.01. On a step without a gradient, the parameters and $\tau_t$ stay as they are. All Muon-based methods in a comparison use the same peak rate $\eta^{\rm aux}_{\max}$ (Section~\ref{sec:experiments} and Table~\ref{tab:extended-settings}), and the AdamW and Magma baselines use their tuned learning rate for all parameters.

\subsection{Schedule}

The schedule of Equation~\eqref{eq:schedule} starts at the peak learning rate with zero warmup and reaches zero at the final step. The row factors $h_{er,t}=1+\lambda_t(q_{er,t}-1)$ therefore equal $q_{er,t}$ at the first step and one at the final step, while the family factor $c_{e,t}$ uses the current expert scores at every step. The momentum is also updated at the final step, where the parameter update and the decay are zero.

\subsection{Properties of the scalar factors}
\label{sec:properties}

The following properties hold in exact arithmetic with the constants defined above.

\paragraph{Family mean.}
For a family with $E\epsilon_{\rm score}\le1$, the family factors in Equation~\eqref{eq:expertgate} satisfy
\begin{equation}
c_{e,t}>0,\qquad \frac1E\sum_e c_{e,t}=1.
\label{eq:family-mean-property}
\end{equation}
Since $a_{e,t}\ge\epsilon_{\rm score}$, the mean of $a$ is at least $\epsilon_{\rm score}$, so $u_{e,t}$ has mean one and at least one entry is at least one. After capping, at least one entry has $v_{e,t}=1$, so $E^{-1}\sum_e v_{e,t}\ge1/E\ge\epsilon_{\rm score}$. The final denominator is therefore the mean of $v$, which gives Equation~\eqref{eq:family-mean-property}. All families in our models satisfy the condition, since $E\le64$. Individual factors can exceed one, and the total norm of the expert updates also depends on their matrix norms and radii.

\paragraph{Radius bounds.}
For $0\le\lambda_t\le1$, the row construction yields $0<q_{er,t}\le1$ and $1-\lambda_t\le h_{er,t}\le1$. Define
\begin{equation}
S_e=\sum_r\lVert O_{er,t}\rVert_2^2,\qquad
H_e=\sum_r h_{er,t}^2\lVert O_{er,t}\rVert_2^2.
\end{equation}
Since $0\le H_e\le S_e$, Equation~\eqref{eq:extended-radius} gives
\begin{equation}
0\le\alpha_{e,t}\le1.
\label{eq:radius-bounds}
\end{equation}
Above the floor in the denominator of Equation~\eqref{eq:extended-radius}, $(1-\lambda_t)^2S_e\le H_e\le S_e$ and hence $1-\lambda_t\le\alpha_{e,t}\le1$. The product $c_{e,t}\alpha_{e,t}$ can exceed one, since $c_{e,t}$ can exceed one.

\paragraph{Norm and direction.}
Above the floor,
\begin{align}
\lVert\alpha_{e,t}O_{e,t}\rVert_F^2
&=\alpha_{e,t}^2S_e\\
&=H_e\\
&=\lVert\operatorname{diag}(h_{e,t})O_{e,t}\rVert_F^2,
\label{eq:radius-energy-proof}
\end{align}
where $\operatorname{diag}(h_{e,t})$ is the diagonal matrix of row factors. The matrix $\alpha_{e,t}O_{e,t}$ therefore has the norm of the row-weighted matrix and the direction of $O_{e,t}$ (Figure~\ref{fig:geometry}).

\begin{figure}[ht]
\centering
\includegraphics[width=0.5\linewidth]{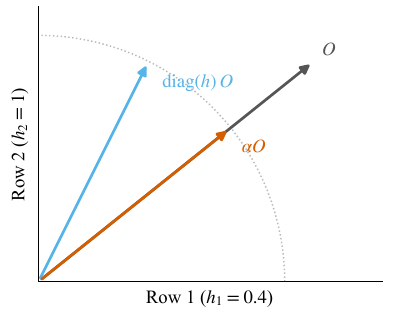}
\caption{\textbf{The scalar radius for a matrix with two rows} and row factors $h=(0.4,1)$. Row weighting gives $\operatorname{diag}(h)\,O$ and changes the direction. Compass applies $\alpha O$, which has the norm of $\operatorname{diag}(h)\,O$ (dotted arc) and the direction of $O$. The family factor is omitted.}
\label{fig:geometry}
\end{figure}

Let $O_{e,t}=L\Sigma R^\top$ be a singular value decomposition. For $\kappa=c_{e,t}\alpha_{e,t}>0$,
\begin{equation}
V_{e,t}=sL(\kappa\Sigma)R^\top,\qquad
\frac{V_{e,t}}{\lVert V_{e,t}\rVert_F}
=\frac{O_{e,t}}{\lVert O_{e,t}\rVert_F},
\label{eq:direction-property}
\end{equation}
so the update keeps the singular vectors of $O_{e,t}$ and the ratios between its singular values. The weight change also includes weight decay.

\paragraph{Negative alignment.}
A matrix whose direction has a negative cosine with its gradient receives the score $\epsilon_{\rm score}$. When the other experts in its family have scores well above $\epsilon_{\rm score}$, its family factor and therefore its step are close to zero.

\FloatBarrier\bigskip
\section{Comparison with heavy-ball Muon at two widths}
\label{sec:width}
\label{sec:experiment_details}

These runs train six-layer mixture-of-experts models from random initialization on FineWeb-Edu, with heavy-ball Muon as the baseline. They have their own learning-rate sweeps, so we report them separately from Section~\ref{sec:experiments}. The two sets of runs share seed numbers and selected learning rates. These runs used H100 GPUs, and their losses differ slightly from those of the same seeds in Table~\ref{tab:current-width}. Training windows come from the training part of a 60M-token sample of the \texttt{sample-10BT} subset, and both methods use the settings of Table~\ref{tab:settings}.

\begin{table}[ht]
\centering
\caption{Settings of the runs at the two widths. Both methods use the same values.}
\label{tab:settings}
\small
\begin{tabular}{@{}lrr@{}}
\toprule
Setting & Width 1,024 & Width 2,048\\
\midrule
Layers & 6 & 6\\
Attention heads & 16 & 16\\
Experts per layer & 8 & 8\\
Selected experts per token & 2 & 2\\
Expert hidden width & 2,048 & 4,096\\
Sequences per step & 16 & 16\\
Sequence length & 256 & 256\\
Steps & 2,000 & 2,000\\
Training tokens per run & 8,192,000 & 8,192,000\\
Validation tokens per evaluation & 32,768 & 32,768\\
Evaluation seeds & 8 & 8\\
\bottomrule
\end{tabular}
\end{table}

Both methods use auxiliary-loss-free load balancing \citep{wang2024lossfree}. The router and its bias updates, the architecture, the FP32 arithmetic, the AdamW learning rate of 0.003, and the validation batches match Appendix~\ref{sec:new-details}, and the schedule follows Equation~\eqref{eq:schedule} with $T=2000$. The bias updates are independent of the Compass factors.

We selected the learning rates of both methods (Table~\ref{tab:width-results}) on tuning seeds before running the evaluation seeds.

\subsection{Metrics and tests}
\label{sec:statistics}

We use the paired statistics from Appendix~\ref{sec:new-statistics}, with heavy-ball Muon as the baseline; a win is a seed on which Compass has the lower final validation loss. With eight seeds, the exact permutation test has $2^8=256$ sign assignments, so its smallest two-sided $p$-value is $2/256\approx0.0078$, the value at both widths in Table~\ref{tab:width-results}. Figure~\ref{fig:widthpairs} shows each seed.

\begin{table}[ht]
\centering
\caption{\textbf{Comparison with heavy-ball Muon at two widths on FineWeb-Edu.} Final validation loss over eight evaluation seeds, with 8.192M training tokens per run. Values are mean $\pm$ sample standard deviation, and $\Delta$ is Compass minus heavy-ball Muon with its 95\% confidence interval and statistic from a paired $t$-test. Heavy-ball Muon and Compass use learning rates 0.001125 and 0.00125 at width 1,024 and 0.0006875 and 0.00075 at width 2,048, selected on separate tuning seeds. The last column is the exact two-sided paired permutation $p$-value.}
\label{tab:width-results}
\fontsize{8.5}{10.5}\selectfont
\setlength{\tabcolsep}{2.8pt}
\begin{tabular}{@{}rrrrrrrr@{}}
\toprule
Width & Heavy-ball Muon & Compass & $\Delta$ & 95\% CI & $t$ & Wins & $p$ \\
\midrule
1,024 & $5.4661\pm0.0078$ & $5.4551\pm0.0075$ & $-0.0110\pm0.0065$ & $[-0.01638,-0.00555]$ & $-4.79$ & 8/8 & 0.0078 \\
2,048 & $5.3291\pm0.0114$ & $5.3192\pm0.0094$ & $-0.0099\pm0.0049$ & $[-0.01393,-0.00578]$ & $-5.72$ & 8/8 & 0.0078 \\
\bottomrule
\end{tabular}
\end{table}

\begin{figure}[ht]
\centering
\includegraphics[width=\linewidth]{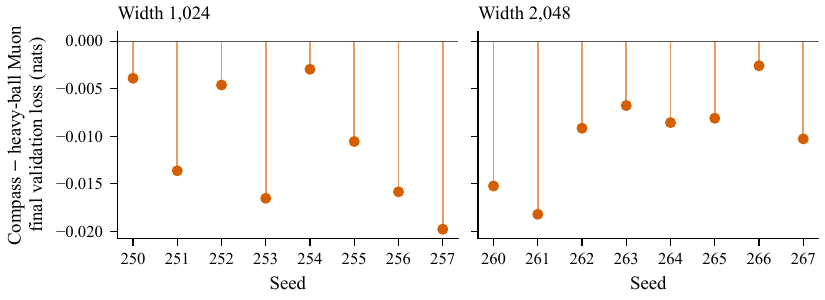}
\caption{\textbf{Per-seed differences in the heavy-ball comparison.} Final validation loss of Compass minus heavy-ball Muon for each evaluation seed. All 16 seeds favor Compass.}
\label{fig:widthpairs}
\end{figure}

Figure~\ref{fig:widthdevelopment} shows the sweep over eight learning rates at width 2,048 on tuning seeds 2600 to 2604. For both methods, the best rate on these seeds is the one we used for the evaluation seeds. At these rates, Compass is lower on four of the five tuning seeds, by an average of 0.00754 nats.

\begin{figure}[ht]
\centering
\includegraphics[width=0.7\linewidth]{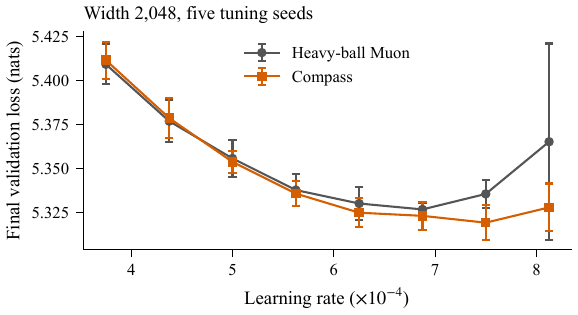}
\caption{\textbf{Learning-rate sweep at width 2,048.} Points are means over five tuning seeds, and error bars are one sample standard deviation.}
\label{fig:widthdevelopment}
\end{figure}

\FloatBarrier\bigskip
\section{Continued pretraining of JetMoE}
\label{sec:extended-setup}

We apply Compass to JetMoE-8B, a released MoE checkpoint whose attention layers and feedforward layers both contain experts \citep{shen2024jetmoe}. Three runs with different seeds lower its validation loss on FineWeb-Edu by about 0.043 nats after 8.2M training tokens (Table~\ref{tab:jetmoe-results}). These runs show that Compass applies to a released model of this size.

\subsection{Model, data and settings}

JetMoE-8B has expert attention layers and expert feedforward layers, with about 8B total parameters and 2B active parameters per token \citep{shen2024jetmoe}. The runs start from the released JetMoE-8B checkpoint with a new optimizer state, and Compass updates both kinds of expert matrices. Packed expert tensors are split along the expert axis, so the last two axes form each expert matrix, and a family groups the matching projections within a layer and module. Router matrices use heavy-ball Muon, and embeddings, the output head, biases, and normalization parameters use AdamW.

The data are the \texttt{sample-10BT} subset of FineWeb-Edu, using the model's released tokenizer. Documents are split into training, validation, and test data in proportions 98\%, 1\%, and 1\% before packing and are separated by an end-of-sequence token. The context length is 2,048, the training order is set by the seed, and the evaluation order uses seed 314159. Each validation evaluation scores 64 batches of eight sequences, for $64\times8\times2047=1{,}048{,}064$ target tokens, and the validation loss is the mean negative log-likelihood over these tokens. The best checkpoint of a run is the one with the lowest validation loss.

\begin{table}[ht]
\centering
\caption{Settings of the JetMoE runs. Learning rates are peak values.}
\label{tab:extended-settings}
\small
\begin{tabular}{@{}ll@{}}
\toprule
Setting & Value \\
\midrule
Parameters, gradients and optimizer states & FP32 \\
GPUs per run & $2\times$H100 \\
Tokens per step & 8,192 \\
Sequence length & 2,048 \\
Accumulated micro-batches per GPU & 2 \\
Sequences per micro-batch per GPU & 1 \\
Steps & 1,000 \\
Tokens per run & 8,192,000 \\
Matrix learning rate & 0.0001 \\
AdamW learning rate & 0.00002 \\
Learning-rate schedule & Cosine, Equation~\eqref{eq:schedule}, zero warmup \\
Gradient clipping & Global norm 1 \\
Muon momentum & 0.95 \\
Weight decay & 0.01 \\
Validation interval (steps) & 100 \\
Seeds & 991, 997, 1009 \\
Software & Python 3.11, PyTorch 2.10.0, CUDA 12.8 \\
 & Transformers 4.57.6, NumPy 1.26.4 \\
\bottomrule
\end{tabular}
\end{table}

Gradients are accumulated over micro-batches, averaged across the two GPUs in FP32 and then clipped, and automatic mixed precision is disabled. To save GPU memory, we update each expert family on one GPU, keep its momentum in host memory between steps, and broadcast the new parameters to the other GPU; the update is unchanged. Attention uses the deterministic implementation with TF32 off, activation checkpointing is on, and the update follows the order of operations in Appendix~\ref{sec:extended-arithmetic}.

\subsection{Per-seed results}
\label{sec:jetmoe-details}

\begin{table}[ht]
\centering
\caption{\textbf{Compass on JetMoE-8B.} Validation loss on FineWeb-Edu after 1,000 steps and 8.192M training tokens, with AdamW learning rate 0.00002. The first three rows use matrix learning rate 0.0001, and the last row is an additional run at 0.00005. The initial validation loss is 2.157508.}
\label{tab:jetmoe-results}
\small
\begin{tabular}{@{}lrrr@{}}
\toprule
Seed & Matrix learning rate & Best loss & Final loss \\
\midrule
991 & 0.0001 & 2.104669 & 2.115326 \\
997 & 0.0001 & 2.102527 & 2.115145 \\
1009 & 0.0001 & 2.104960 & 2.113939 \\
\midrule
Mean $\pm$ SD & 0.0001 & $2.104052\pm0.001329$ & $2.114804\pm0.000754$ \\
\midrule
991 & 0.00005 & 2.072954 & 2.075305 \\
\bottomrule
\end{tabular}
\end{table}

All three seeds in Table~\ref{tab:jetmoe-results} reach their lowest validation loss at step 100. An additional run with half the matrix learning rate (0.00005, seed 991) reaches a lower loss, shown in the last row of the table. A loss difference $\delta$ corresponds to a perplexity change of $100(1-e^{\delta})$ percent.

\FloatBarrier\bigskip
\section{Ablations of the alignment scores}
\label{sec:earlier-alignment}

These runs vary how Compass alignment scores affect the update, with heavy-ball Muon as a baseline. They use a separate data split, H100 GPUs, and FP32 arithmetic, so absolute losses differ from those in Section~\ref{sec:experiments}; heavy-ball Muon reaches 5.352 at width 1,024 here.

\paragraph{Setup.}
The models have six layers, eight experts per layer, and top-2 routing, and they train with context length 256, 16 sequences per step, and 2,000 steps (8.192M training tokens). Widths 1,024, 1,536, 2,048, and 2,560 give 329.8M, 664.5M, 1,112.5M, and 1,673.8M total parameters. We train on a 60M-token FineWeb-Edu sample (95\% training and 5\% validation tokens) and on a DCLM-Baseline sample of the same size \citep{li2024dclm}. Parameters, gradients, and optimizer states are in FP32, and each run uses one NVIDIA H100 GPU with PyTorch 2.7. Heavy-ball Muon has the shape factor $0.2\sqrt{\max(m,n)}$ and the AdamW settings of the variants. We tuned its learning rate on a grid at each width and obtained 0.00125 at width 1,024 on both corpora and 0.000625 at the larger widths (the width-1,536 grid had the two rates 0.000625 and 0.00125). All variants use this rate. Variants that we also swept over two or three learning rates on seeds 0 to 2 picked the same rate. The heavy-ball grids include seeds that the comparisons also use. This favors the baseline.

\paragraph{Variants.}
The family-factor variant multiplies the orthogonalized update of expert $e$ by $c_e$ (Equation~\eqref{eq:expertgate}) with the radius set to one, and its uncapped form uses $u_e$. Row weighting multiplies row $r$ of $O_e$ by $q_{er}/\overline q_e$ (Equation~\eqref{eq:rowgate} with $\lambda_t=1$). This changes the direction of the update, and the scalar radius of Compass keeps the direction. The combined variant (family $\times$ row) applies both. In these runs every matrix uses heavy-ball momentum, so $O_e$ comes from $M_e$, and Excitation follows its published utilization-based equations. The controls replace the score with the gradient norm, the momentum cosine $\cos(M_e, G_e)$ or the reciprocal of $a_e$, permute the capped factors among the experts of each family at every step, or drop the mean-one normalization, which multiplies each expert step by $a_e$.

\begin{table}[t]
\centering
\caption{\textbf{Ablations of the alignment scores against heavy-ball Muon.} Final validation loss of each variant minus heavy-ball Muon at its selected learning rate; negative values favor the variant, and parentheses give wins over the number of seeds. Top: model widths and corpora; only settings with runs have entries. Bottom: variants and controls at width 1,024 on FineWeb-Edu with seeds 0 to 7, with 95\% confidence intervals from paired $t$-tests. The scores $a_e$ and $b_{er}$ are the positive parts of the matrix and row cosines plus $10^{-6}$ (Section~\ref{sec:method}), and Cap marks capping at the family or row mean.}
\label{tab:earlier-alignment}
\footnotesize
\setlength{\tabcolsep}{3pt}
\begin{tabular}{@{}lrrrrr@{}}
\toprule
& \multicolumn{4}{c}{FineWeb-Edu, width} & DCLM \\
\cmidrule(lr){2-5}\cmidrule(l){6-6}
Variant & 1,024 & 1,536 & 2,048 & 2,560 & 1,024 \\
\midrule
Family factor & $-0.0057$ (11/16) & $-0.0041$ (5/8) & $-0.0049$ (6/8) & $-0.0070$ (7/8) & $-0.0041$ (6/8) \\
Family, uncapped & $-0.0039$ (13/16) & $-0.0039$ (7/8) & $-0.0064$ (9/9) & $+0.0175$ (5/8) & \\
Family $\times$ row & $-0.0059$ (14/16) & & $-0.0078$ (8/8) & & $-0.0075$ (6/8) \\
Row weighting & $-0.0034$ (11/16) & & $-0.0010$ (3/8) & & $-0.0069$ (7/8) \\
Excitation & $+0.0035$ (2/8) & & $-0.0080$ (7/8) & & $+0.0029$ (2/8) \\
\bottomrule
\end{tabular}

\vspace{6pt}
\begin{tabular}{@{}llcr@{}}
\toprule
Level & Score & Cap & Difference [95\% CI] (wins) \\
\midrule
Family & $a_e$ & yes & $-0.00753\;[-0.01391,-0.00114]$ (6/8) \\
Family & $a_e$ & no & $-0.00393\;[-0.00896,+0.00109]$ (7/8) \\
Family & $a_e$, factors permuted & yes & $-0.00066\;[-0.00457,+0.00326]$ (5/8) \\
Family & $\norm{G_e}_F$ & yes & $-0.00336\;[-0.01035,+0.00363]$ (5/8) \\
Family & $[\cos(M_e,G_e)]_+$ & no & $-0.00419\;[-0.00866,+0.00028]$ (6/8) \\
Family & $1/a_e$ & no & $+0.01812\;[+0.00973,+0.02651]$ (1/8) \\
Family & $a_e$ without mean-one normalization & & $+0.19909\;[+0.19317,+0.20500]$ (0/8) \\
Row & $b_{er}$ & yes & $-0.00482\;[-0.01016,+0.00052]$ (7/8) \\
Row & $b_{er}$ & no & $-0.00871\;[-0.01344,-0.00398]$ (8/8) \\
Row & $1/b_{er}$ & no & $+0.45838\;[+0.44986,+0.46691]$ (0/8) \\
\bottomrule
\end{tabular}
\end{table}

\paragraph{Results.}
Table~\ref{tab:earlier-alignment} lists every difference. The combined variant lowers the loss in all three settings where we ran it, and its 95\% confidence intervals lie below zero in each ($[-0.00883,-0.00302]$ at width 1,024, $[-0.01346,-0.00209]$ at width 2,048, and $[-0.01414,-0.00090]$ on DCLM). The family factor alone lowers the mean loss at all four widths and on DCLM, with intervals below zero at widths 1,024 and 2,560 ($[-0.00929,-0.00203]$ and $[-0.01058,-0.00341]$). The uncapped family factor lowers the loss up to width 2,048. At width 2,560, two of its seeds end 0.068 and 0.109 nats above heavy-ball Muon, and with the cap the same seeds change by $+0.0001$ and $-0.0127$. Row weighting alone lowers the loss at width 1,024 on both corpora. Excitation has a higher mean loss than heavy-ball Muon at width 1,024 on both corpora and comes close to the combined variant at width 2,048 (95\% CI $[-0.01646,0.00049]$).

The controls at width 1,024 use seeds 0 to 7 (bottom of Table~\ref{tab:earlier-alignment}). Permuting the capped factors among the experts of each family removes most of the benefit of the family factor, and the reciprocal score or dropping the mean-one normalization raises the loss sharply. The cosine factor is 0.00417 nats lower than a factor computed from gradient norms (95\% CI $[-0.00757,-0.00076]$). In the uncapped form, the momentum cosine $\cos(M_e,G_e)$ and the update cosine $\cos(O_e,G_e)$ give nearly equal results (difference 0.00026 nats, 95\% CI $[-0.00299,0.00350]$). At the row level, both weightings lower the loss, and the reciprocal row score raises it.

\paragraph{Relation to Compass.}
Compass keeps the capped family factor and turns the capped row factors into a single scalar radius, preserving the orthogonalized direction that row weighting changes. Across these runs, alignment-based scaling lowers the loss against tuned heavy-ball Muon at both the expert and row levels, while the permuted and reciprocal controls remove the gain or raise the loss.

\clearpage
\section{Theory}
\label{sec:cth-proofs}

We state the setting, prove the perturbation bound of Theorem~\ref{thm:cth-stability}, convert gradient errors into score errors, and, study the Newton-Schulz map.

\subsection{Setting}
\label{sec:cth-domain}
We use the Frobenius inner product and norm of Appendix~\ref{sec:numerics} for matrices and the Euclidean ones for rows. For a matrix $X$, $X_r$ is its $r$th row. For a vector $h$, $\|h\|_\infty=\max_r|h_r|$, and $\operatorname{diag}(h)$ is the diagonal matrix with diagonal $h$. The proofs rely on the Cauchy-Schwarz inequality $|\langle X,Y\rangle|\leq\|X\|\,\|Y\|$, the triangle inequality, and the reverse triangle inequality $\bigl|\|X\|-\|Y\|\bigr|\leq\|X-Y\|$ \citep[Chapter~5]{horn2013matrix}.

We keep the constants $\epsilon_{\rm cos}$, $\epsilon_{\rm score}$ and $\epsilon_{\rm NS}$ of Table~\ref{tab:method-symbols}, and we write the floor $10^{-30}$ of Equation~\eqref{eq:extended-radius} as $\tau^2$ with $\tau=10^{-15}$. The results hold for any positive values of these constants with $E\epsilon_{\rm score}\leq1$ and for any $\tau\geq0$. The case $\tau=0$ is the radius of Equation~\eqref{eq:radius}, without a floor. With the cosine $\operatorname{cos}_\epsilon(X,Y)=\langle X,Y\rangle/(\|X\|\,\|Y\|+\epsilon_{\rm cos})$ of Equation~\eqref{eq:exact-cosine}, the scores are
\begin{equation}
 a_e=[\operatorname{cos}_\epsilon(O_e,G_e)]_++\epsilon_{\rm score},\qquad
 b_{er}=[\operatorname{cos}_\epsilon(O_{er},G_{er})]_++\epsilon_{\rm score}.
 \label{eq:cth-scores}
\end{equation}
A bar denotes the mean over the family, or over the rows of one matrix for the row scores. With the capped score $\check a_e=\min(a_e,\overline a)$, its mean $\mu=\overline{\check a}$, and the row means $\nu_e=\overline b_e$, the update is
\begin{equation}
 c_e=\frac{\check a_e}{\mu},\qquad
 q_{er}=\min\Bigl(1,\frac{b_{er}}{\nu_e}\Bigr),\qquad
 h_{er}=1+\lambda(q_{er}-1),\qquad
 V_e=sc_e\Phi_{\tau,h_e}(O_e),
 \label{eq:cth-literal}
\end{equation}
where $\Phi_{\tau,h}(0)=0$ and
\begin{equation}
 \Phi_{\tau,h}(O)=\frac{\|\operatorname{diag}(h)O\|_F}{\max(\|O\|_F,\tau)}\,O\qquad(O\neq0).
 \label{eq:cth-phi}
\end{equation}

This is the update of Section~\ref{sec:method} and Appendix~\ref{sec:extended-arithmetic}. In Equation~\eqref{eq:expertgate}, $v_e=\min(1,a_e/\overline a)=\check a_e/\overline a$, so $\overline v=\mu/\overline a$ and
\begin{equation}
 c_e=\frac{v_e}{\overline v}=\frac{\check a_e}{\mu},\qquad \mu=\overline a\,\overline v .
 \label{eq:cth-cancel}
\end{equation}
The clamps of Appendix~\ref{sec:extended-arithmetic} leave these values unchanged, because $\overline a\geq\epsilon_{\rm score}$ and $\overline v\geq1/E\geq\epsilon_{\rm score}$ (Appendix~\ref{sec:properties}). For $O\neq0$, Equation~\eqref{eq:extended-radius} gives
\begin{equation}
 \alpha=\Bigl(\frac{\sum_rh_r^2\|O_r\|_2^2}{\max(\|O\|_F^2,\tau^2)}\Bigr)^{1/2}
 =\frac{\|\operatorname{diag}(h)O\|_F}{\max(\|O\|_F,\tau)},
\end{equation}
so $\alpha O=\Phi_{\tau,h}(O)$. For $O=0$, the radius is one and $\alpha O=0=\Phi_{\tau,h}(0)$. Hence $V_e=sc_e\alpha_eO_e=sc_e\Phi_{\tau,h_e}(O_e)$, where $s>0$.

The two updates in Theorem~\ref{thm:cth-stability} use the same family of experts. An expert with a zero gradient belongs to the family. An expert without a gradient in a step does not, because Compass skips it and Compass on the NorMuon update gives it NorMuon's update without Compass factors (Appendix~\ref{sec:new-details}). The matrices may be rectangular or rank deficient.

\subsection{Proof of the perturbation bound}
\label{sec:cth-stability-proof}

\begin{compasslemma}[Bounds on the family and row factors]
\label{lem:cth-normalization}
For both families, $\overline c=1$, $c_e>0$, $0<q_{er}\leq1$ and $0<h_{er}\leq1$. Also $A,B\geq\epsilon_{\rm score}$, and with $\delta_a$ and $\delta_b$ as in Theorem~\ref{thm:cth-stability},
\begin{align}
 |c_e-\widetilde c_e|&\leq\frac{\delta_a}{A}(1+\widetilde c_e),&
 \frac1E\sum_e|c_e-\widetilde c_e|&\leq\frac{2\delta_a}{A},
 \label{eq:cth-cpert}\\
 |q_{er}-\widetilde q_{er}|&\leq\frac{2\delta_b}{B},&
 \|h_e-\widetilde h_e\|_\infty&\leq\frac{2\lambda\delta_b}{B}.
 \label{eq:cth-hpert}
\end{align}
\end{compasslemma}
\begin{proof}
By the Cauchy-Schwarz inequality,
\begin{equation}
 |\operatorname{cos}_\epsilon(X,Y)|\leq\frac{\|X\|\,\|Y\|}{\|X\|\,\|Y\|+\epsilon_{\rm cos}}<1,
 \qquad\text{so}\qquad
 \epsilon_{\rm score}\leq a_e,\,b_{er}<1+\epsilon_{\rm score}.
 \label{eq:cth-score-range}
\end{equation}
Means and minima of numbers in this interval also lie in it. Hence $\overline a$, $\check a_e$, $\mu$ and $\nu_e$ lie in this interval for both families, and $A,B\geq\epsilon_{\rm score}$. Summing $c_e=\check a_e/\mu$ over the family gives
\begin{equation}
 \sum_ec_e=\frac1\mu\sum_e\check a_e=\frac{E\mu}{\mu}=E,
 \label{eq:cth-csum}
\end{equation}
so $\overline c=1$, and $c_e>0$ because $\check a_e>0$. Also $0<q_{er}\leq1$ because $b_{er}/\nu_e>0$, and $h_{er}=(1-\lambda)+\lambda q_{er}$ lies between $q_{er}$ and one because $\lambda\in[0,1]$.

A minimum changes by at most the change of its arguments. If $|x-x'|\leq\delta$ and $|y-y'|\leq\delta$, then $\min(x,y)\leq\min(x'+\delta,y'+\delta)=\min(x',y')+\delta$, and exchanging the two pairs gives
\begin{equation}
 |\min(x,y)-\min(x',y')|\leq\delta .
 \label{eq:cth-min}
\end{equation}
The same argument proves this inequality for the maximum. Averaging $|a_j-\widetilde a_j|\leq\delta_a$ over the family, Equation~\eqref{eq:cth-min} with $x=a_e$ and $y=\overline a$, and averaging again give, in turn,
\begin{equation}
 |\overline a-\overline{\widetilde a}|\leq\delta_a,\qquad
 |\check a_e-\widetilde{\check a}_e|\leq\delta_a,\qquad
 |\mu-\widetilde\mu|\leq\delta_a .
 \label{eq:cth-aerr}
\end{equation}
Since
\begin{equation}
 c_e-\widetilde c_e
 =\frac{\check a_e-\widetilde{\check a}_e}{\mu}+\widetilde{\check a}_e\Bigl(\frac1\mu-\frac1{\widetilde\mu}\Bigr)
 =\frac{\check a_e-\widetilde{\check a}_e}{\mu}+\widetilde c_e\,\frac{\widetilde\mu-\mu}{\mu}
 \label{eq:cth-cdiff}
\end{equation}
and $\mu\geq A$, we get
\begin{equation}
 |c_e-\widetilde c_e|\leq\frac{\delta_a}{\mu}+\widetilde c_e\,\frac{\delta_a}{\mu}\leq\frac{\delta_a}{A}(1+\widetilde c_e).
\end{equation}
Averaging over $e$ with $\overline{\widetilde c}=1$ gives the second bound in Equation~\eqref{eq:cth-cpert}.

For the rows of one matrix, $q_{er}=y_{er}/\nu_e$ with $y_{er}=\min(b_{er},\nu_e)$. As above,
\begin{equation}
 |\nu_e-\widetilde\nu_e|\leq\delta_b,\qquad |y_{er}-\widetilde y_{er}|\leq\delta_b .
\end{equation}
Equation~\eqref{eq:cth-cdiff}, with $y_{er}$, $\nu_e$ and $\widetilde q_{er}$ in place of $\check a_e$, $\mu$ and $\widetilde c_e$, gives
\begin{equation}
 |q_{er}-\widetilde q_{er}|\leq\frac{\delta_b}{B}\,(1+\widetilde q_{er})\leq\frac{2\delta_b}{B},
\end{equation}
because $\nu_e\geq B$ and $\widetilde q_{er}\leq1$. Finally, $h_{er}-\widetilde h_{er}=\lambda(q_{er}-\widetilde q_{er})$.
\end{proof}

\begin{compasslemma}[The radius map]
\label{lem:cth-radial}
For $\tau\geq0$, $h,\widetilde h\in[0,1]^m$ and $X,Y\in\R^{m\times n}$,
\begin{align}
 \|\Phi_{\tau,h}(X)\|_F&\leq\|X\|_F,\label{eq:cth-phi-norm}\\
 \|\Phi_{\tau,h}(X)-\Phi_{\tau,\widetilde h}(Y)\|_F
 &\leq3\|X-Y\|_F+\|Y\|_F\|h-\widetilde h\|_\infty.
 \label{eq:cth-phi-lip}
\end{align}
\end{compasslemma}
\begin{proof}
Let
\begin{equation}
 u_h(X)=\|\operatorname{diag}(h)X\|_F,\qquad
 J(X)=\frac{X}{\max(\|X\|_F,\tau)}\quad(X\neq0),\qquad J(0)=0 .
\end{equation}
Then $\Phi_{\tau,h}=u_hJ$ and $\|J(X)\|_F\leq1$. For any vector $k$,
\begin{equation}
 \|\operatorname{diag}(k)X\|_F^2=\sum_rk_r^2\|X_r\|_2^2\leq\|k\|_\infty^2\|X\|_F^2 .
 \label{eq:cth-rows}
\end{equation}
With $k=h$, this gives $u_h(X)\leq\|X\|_F$ and hence Equation~\eqref{eq:cth-phi-norm}. With $k=h-\widetilde h$ and the reverse triangle inequality, it gives $|u_h(Y)-u_{\widetilde h}(Y)|\leq\|h-\widetilde h\|_\infty\|Y\|_F$, so
\begin{equation}
 \|\Phi_{\tau,h}(Y)-\Phi_{\tau,\widetilde h}(Y)\|_F=|u_h(Y)-u_{\widetilde h}(Y)|\,\|J(Y)\|_F\leq\|h-\widetilde h\|_\infty\|Y\|_F .
\end{equation}
By the triangle inequality, it remains to show that $\|\Phi_{\tau,h}(X)-\Phi_{\tau,h}(Y)\|_F\leq3\delta$ for $\delta=\|X-Y\|_F$. This difference is symmetric in $X$ and $Y$, so we may assume $\rho=\|X\|_F\geq\rho'=\|Y\|_F$. If $Y=0$, the bound follows from Equation~\eqref{eq:cth-phi-norm}, because $\|X\|_F=\delta$. Otherwise, $d=\max(\rho,\tau)$ and $d'=\max(\rho',\tau)$ are positive, and $d\geq d'$. Equation~\eqref{eq:cth-min} for the maximum and the reverse triangle inequality give
\begin{equation}
 0\leq d-d'\leq\rho-\rho'\leq\delta .
\end{equation}
From
\begin{equation}
 J(X)-J(Y)=\frac{X-Y}{d}-\frac{d-d'}{d\,d'}\,Y
\end{equation}
and $\|Y\|_F\leq d'$, we get
\begin{equation}
 \|J(X)-J(Y)\|_F\leq\frac{\delta}{d}+\frac{\|Y\|_F}{d'}\cdot\frac{d-d'}{d}\leq\frac{2\delta}{d}.
\end{equation}
Finally,
\begin{equation}
 \Phi_{\tau,h}(X)-\Phi_{\tau,h}(Y)=\bigl(u_h(X)-u_h(Y)\bigr)J(X)+u_h(Y)\bigl(J(X)-J(Y)\bigr),
\end{equation}
where $|u_h(X)-u_h(Y)|\leq u_h(X-Y)\leq\delta$ and $u_h(Y)\leq\rho'\leq d$. The norm of this difference is therefore at most
\begin{equation}
 \delta+d\cdot\frac{2\delta}{d}=3\delta .\qedhere
\end{equation}
\end{proof}

\begin{proof}[Proof of Theorem~\ref{thm:cth-stability}]
Let $F_e=\Phi_{\tau,h_e}(O_e)$ and $\widetilde F_e=\Phi_{\tau,\widetilde h_e}(\widetilde O_e)$, so that $V_e=sc_eF_e$ and $\widetilde V_e=s\widetilde c_e\widetilde F_e$. Then
\begin{equation}
 V_e-\widetilde V_e=sc_e(F_e-\widetilde F_e)+s(c_e-\widetilde c_e)\widetilde F_e,
\end{equation}
and since $c_e>0$,
\begin{equation}
 \|V_e-\widetilde V_e\|_F\leq sc_e\|F_e-\widetilde F_e\|_F+s|c_e-\widetilde c_e|\,\|\widetilde F_e\|_F .
 \label{eq:cth-thm-split}
\end{equation}
Lemma~\ref{lem:cth-normalization} gives $h_e,\widetilde h_e\in(0,1]^m$. Lemma~\ref{lem:cth-radial} with $X=O_e$ and $Y=\widetilde O_e$, together with Equation~\eqref{eq:cth-hpert}, then gives
\begin{equation}
 \|F_e-\widetilde F_e\|_F\leq3\delta_O+\frac{2\lambda R\delta_b}{B},\qquad
 \|\widetilde F_e\|_F\leq\|\widetilde O_e\|_F\leq R .
\end{equation}
We insert these bounds in Equation~\eqref{eq:cth-thm-split} and average over the family. With $\overline c=1$ and Equation~\eqref{eq:cth-cpert},
\begin{equation}
 \frac1E\sum_e\|V_e-\widetilde V_e\|_F
 \leq s\Bigl(3\delta_O+\frac{2\lambda R\delta_b}{B}\Bigr)+sR\,\frac{2\delta_a}{A},
\end{equation}
which is Equation~\eqref{eq:cth-stability}. Lemma~\ref{lem:cth-normalization} also gives $A,B>0$. The proof compares two fixed sets of inputs, so the bound also holds when the scores depend on the same gradient as the direction.
\end{proof}

\subsection{From gradient errors to score errors}
\label{sec:cth-routing}

\begin{compasslemma}[Errors of the cosine]
\label{lem:cth-cosine}
For vectors or vectorized matrices $x,y,x',y'$ and $\epsilon>0$,
\begin{equation}
 \left|\frac{\langle x,y\rangle}{\|x\|\|y\|+\epsilon}
       -\frac{\langle x',y'\rangle}{\|x'\|\|y'\|+\epsilon}\right|
 \leq\frac{2(\|y\|\|x-x'\|+\|x'\|\|y-y'\|)}{\|x\|\|y\|+\epsilon}.
 \label{eq:cth-cosine-lip}
\end{equation}
Taking positive parts and adding $\epsilon_{\rm score}$ to both cosines does not increase the left side.
\end{compasslemma}
\begin{proof}
Let $N=\langle x,y\rangle$, $P=\|x\|\|y\|$, $N'=\langle x',y'\rangle$, $P'=\|x'\|\|y'\|$ and $\Delta=\|y\|\|x-x'\|+\|x'\|\|y-y'\|$. The identities
\begin{equation}
 N-N'=\langle x-x',y\rangle+\langle x',y-y'\rangle,\qquad
 P-P'=(\|x\|-\|x'\|)\|y\|+\|x'\|(\|y\|-\|y'\|)
\end{equation}
and the Cauchy-Schwarz and reverse triangle inequalities give
\begin{equation}
 |N-N'|\leq\Delta,\qquad |P-P'|\leq\Delta .
\end{equation}
Since
\begin{equation}
 \frac{N}{P+\epsilon}-\frac{N'}{P'+\epsilon}
 =\frac{N-N'}{P+\epsilon}+\frac{N'}{P'+\epsilon}\cdot\frac{P'-P}{P+\epsilon}
\end{equation}
and $|N'|\leq P'$, the left side of Equation~\eqref{eq:cth-cosine-lip} is at most
\begin{equation}
 \frac{\Delta}{P+\epsilon}+\frac{\Delta}{P+\epsilon}=\frac{2\Delta}{P+\epsilon}.
\end{equation}
The positive part $[t]_+=\max(t,0)$ satisfies
\begin{equation}
 |[t]_+-[t']_+|\leq|t-t'|
\end{equation}
by Equation~\eqref{eq:cth-min} for the maximum, and adding the same constant to two numbers does not change their difference.
\end{proof}

Lemma~\ref{lem:cth-cosine} with $x=O_e$, $y=G_e$, $x'=\widetilde O_e$, $y'=\widetilde G_e$ and $\epsilon=\epsilon_{\rm cos}$ gives
\begin{equation}
 |a_e-\widetilde a_e|
 \leq\frac{2\bigl(\|G_e\|_F\|O_e-\widetilde O_e\|_F+\|\widetilde O_e\|_F\|G_e-\widetilde G_e\|_F\bigr)}{\|O_e\|_F\|G_e\|_F+\epsilon_{\rm cos}} .
\end{equation}
For the rows, the lemma and the bounds $\|O_{er}-\widetilde O_{er}\|_2\leq\|O_e-\widetilde O_e\|_F$ and $\|G_{er}-\widetilde G_{er}\|_2\leq\|G_e-\widetilde G_e\|_F$ give
\begin{equation}
 |b_{er}-\widetilde b_{er}|
 \leq\frac{2\bigl(\|G_{er}\|_2\|O_e-\widetilde O_e\|_F+\|\widetilde O_{er}\|_2\|G_e-\widetilde G_e\|_F\bigr)}{\|O_{er}\|_2\|G_{er}\|_2+\epsilon_{\rm cos}} .
\end{equation}
The maxima of these bounds over $e$ and $r$ bound $\delta_a$ and $\delta_b$. For the direction, we fix the momentum $M_{e,t-1}$ of the previous step. Equation~\eqref{eq:expanded-refresh} gives $D_e-\widetilde D_e=(1+\beta)(G_e-\widetilde G_e)$, and Proposition~\ref{prop:cth-polynomial} below gives
\begin{equation}
 \|O_e-\widetilde O_e\|_F
 \leq\frac{2(4.25)^5(1+\beta)\|G_e-\widetilde G_e\|_F}{\min(\|D_e\|_F,\|\widetilde D_e\|_F)+\epsilon_{\rm NS}},
\end{equation}
which bounds $\delta_O$. With these bounds, Theorem~\ref{thm:cth-stability} bounds the error of the update by the errors of the clipped gradients, for any clipping rule. The bound includes the dependence of each factor on its own gradient.

The size of the bound depends on the norms and the scores. For example, $A=B=0.2$, $R=1$, $\delta_O=0$, $\delta_a=\delta_b=10^{-3}$ and $\lambda=1$ give $s(0.01+0.01)=0.02s$ in Equation~\eqref{eq:cth-stability}. Replacing $A$ and $B$ by their lower bound $\epsilon_{\rm score}$ gives a valid but much larger bound.

\subsection{The Newton-Schulz map}
\label{sec:cth-polynomial}

\begin{compassproposition}[Five Newton-Schulz iterations]
\label{prop:cth-polynomial}
Let $p(x)=3.4445x-4.7750x^3+2.0315x^5$ be the polynomial of Equation~\eqref{eq:ns}, and let $P_5=p\circ p\circ p\circ p\circ p$. Let $D\in\R^{m\times n}$ have the singular value decomposition $D=L\Sigma R^\top$ \citep[Chapter~2]{horn2013matrix} with singular values $\sigma_i$, and let $\hat\sigma_i=\sigma_i/(\|D\|_F+\epsilon_{\rm NS})$. Then
\begin{equation}
 \NS(D)=L\operatorname{diag}\bigl(P_5(\hat\sigma_i)\bigr)R^\top,
 \label{eq:cth-spectral}
\end{equation}
where $0\leq P_5(x)\leq1.21$ for $0\leq x\leq1.21$ and $P_5(x)>0$ for $x>0$. Hence $\NS(D)$ has the rank of $D$, $\|\NS(D)\|_{\rm op}\leq1.21$ and $\|\NS(D)\|_F\leq1.21\sqrt{\operatorname{rank}(D)}$. For $D'\in\R^{m\times n}$,
\begin{equation}
 \|\NS(D)-\NS(D')\|_F
 \leq\frac{2(4.25)^5\|D-D'\|_F}{\min(\|D\|_F,\|D'\|_F)+\epsilon_{\rm NS}}.
 \label{eq:cth-nslip}
\end{equation}
\end{compassproposition}
\begin{proof}
Consider first $m\leq n$, where no transpose is taken. Let $X=L\operatorname{diag}(x_i)R^\top$ with orthogonal $L$ and $R$, where $\operatorname{diag}(x_i)$ is the $m\times n$ matrix with diagonal entries $x_i$. Then
\begin{equation}
 XX^\top=L\operatorname{diag}(x_i^2)L^\top,
\end{equation}
and one iteration of Equation~\eqref{eq:ns} maps $X$ to
\begin{equation}
 F(X)=3.4445X-4.7750(XX^\top)X+2.0315(XX^\top)^2X=L\operatorname{diag}\bigl(p(x_i)\bigr)R^\top .
 \label{eq:cth-oneiter}
\end{equation}
The first iterate is
\begin{equation}
 X_0=\frac{D}{\|D\|_F+\epsilon_{\rm NS}}=L\operatorname{diag}(\hat\sigma_i)R^\top,
\end{equation}
so five iterations give Equation~\eqref{eq:cth-spectral}. For $m>n$, the map runs on $D^\top=R\Sigma^\top L^\top$, and transposing the result gives the same formula. Repeated and zero singular values are allowed.

We write
\begin{equation}
 p(x)=x\,k(x^2),\qquad k(y)=2.0315y^2-4.7750y+3.4445 .
\end{equation}
The discriminant of $k$ is
\begin{equation}
 4.7750^2-4(2.0315)(3.4445)=-5.189382<0,
\end{equation}
so $k>0$ and $p(x)>0$ for $x>0$. The derivative is
\begin{equation}
 p'(x)=g(x^2),\qquad g(y)=10.1575y^2-14.325y+3.4445,
\end{equation}
and the discriminant of $g$ is
\begin{equation}
 14.325^2-4(10.1575)(3.4445)=65.25559>0 .
\end{equation}
The two roots $y_-<y_+$ of $g$ are positive because their sum and product are positive. The leading coefficient of $g$ is positive, so $g<0$ exactly between the roots. On $[0,\infty)$, $p$ therefore increases up to $\xi=\sqrt{y_-}$, decreases from $\xi$ to $\sqrt{y_+}$, and increases after $\sqrt{y_+}$. Hence
\begin{equation}
 p(x)\leq\max\{p(\xi),p(1.21)\}\qquad(0\leq x\leq1.21).
 \label{eq:cth-pmax}
\end{equation}
Exact evaluation gives $p'(0.5545)>0>p'(0.5546)$, so $\xi\in(0.5545,0.5546)$. The positive terms of $p$ increase with $x$ and the negative term decreases, so
\begin{equation}
 p(\xi)\leq3.4445(0.5546)-4.7750(0.5545)^3+2.0315(0.5546)^5=1.2028\ldots<1.21 .
\end{equation}
Since also $p(1.21)=0.9778\ldots<1.21$, $p$ maps $[0,1.21]$ into itself, and so does $P_5$. Together with $p(x)>0$ for $x>0$, this proves the bounds on $P_5$. By Equation~\eqref{eq:cth-spectral}, the singular values of $\NS(D)$ are the numbers $P_5(\hat\sigma_i)$. They lie in $[0,1.21]$ because $0\leq\hat\sigma_i<1$ (each $\sigma_i\leq\|D\|_F$), and they are positive exactly when $\sigma_i>0$. This gives the rank and the two norm bounds.

For Equation~\eqref{eq:cth-nslip}, we first bound $p'$. Since $g$ is convex,
\begin{equation}
 \max_{0\leq y\leq1.21^2}g(y)=\max\{g(0),g(1.21^2)\}=\max\{3.4445,\,p'(1.21)\}=4.2448\ldots,
\end{equation}
and
\begin{equation}
 \min_{y\in\R}g(y)=3.4445-\frac{14.325^2}{4\cdot10.1575}=-1.6061\ldots .
\end{equation}
Hence $|p'|\leq4.25$ on $[-1.21,1.21]$, and the mean value theorem \citep[Chapter~5]{rudin1976principles} gives
\begin{equation}
 |p(x)-p(x')|\leq4.25\,|x-x'|\qquad(x,x'\in[-1.21,1.21]).
 \label{eq:cth-plip}
\end{equation}
Let $K_1=Q_1\Theta Q_1^\top$ and $K_2=Q_2\Gamma Q_2^\top$ be symmetric matrices with orthogonal $Q_1,Q_2$ and eigenvalues $\theta_i,\gamma_j\in[-1.21,1.21]$ \citep[Chapter~2]{horn2013matrix}. Since $K_1^k=Q_1\Theta^kQ_1^\top$ for every $k$,
\begin{equation}
 p(K_1)=Q_1p(\Theta)Q_1^\top,\qquad p(K_2)=Q_2p(\Gamma)Q_2^\top .
\end{equation}
With $C=Q_1^\top Q_2$, the matrices
\begin{equation}
 Q_1^\top\bigl(p(K_1)-p(K_2)\bigr)Q_2=p(\Theta)C-Cp(\Gamma),\qquad
 Q_1^\top(K_1-K_2)Q_2=\Theta C-C\Gamma
\end{equation}
have the entries $(p(\theta_i)-p(\gamma_j))C_{ij}$ and $(\theta_i-\gamma_j)C_{ij}$. Orthogonal factors do not change the Frobenius norm, so Equation~\eqref{eq:cth-plip} gives
\begin{equation}
 \|p(K_1)-p(K_2)\|_F\leq4.25\,\|K_1-K_2\|_F .
 \label{eq:cth-symlip}
\end{equation}

For $X\in\R^{m\times n}$, the symmetric matrix
\begin{equation}
 \mathcal D(X)=\begin{pmatrix}0&X\\X^\top&0\end{pmatrix}
\end{equation}
satisfies
\begin{equation}
 \mathcal D(X)^3=\mathcal D(XX^\top X),\qquad \mathcal D(X)^5=\mathcal D\bigl((XX^\top)^2X\bigr),
\end{equation}
and $\mathcal D(X)^2$ is block diagonal with the blocks $XX^\top$ and $X^\top X$. Since $\mathcal D$ is linear,
\begin{equation}
 p(\mathcal D(X))=\mathcal D(F(X)),\qquad \|\mathcal D(X)-\mathcal D(Y)\|_F=\sqrt2\,\|X-Y\|_F .
\end{equation}
An eigenvalue $\theta$ of $\mathcal D(X)$ satisfies $\theta^2\leq\|X\|_{\rm op}^2$, because $\theta^2$ is an eigenvalue of $\mathcal D(X)^2$, whose eigenvalues are those of $XX^\top$ and $X^\top X$. If $\|X\|_{\rm op}\leq1.21$ and $\|Y\|_{\rm op}\leq1.21$, Equation~\eqref{eq:cth-symlip} therefore gives
\begin{equation}
 \|F(X)-F(Y)\|_F=\tfrac1{\sqrt2}\|p(\mathcal D(X))-p(\mathcal D(Y))\|_F\leq4.25\,\|X-Y\|_F .
 \label{eq:cth-rectlip}
\end{equation}
The $k$th iterates $X_k$ of $D$ and $X_k'$ of $D'$ have singular values in $[0,1.21]$, by the first part of the proof. Five applications of Equation~\eqref{eq:cth-rectlip} give
\begin{equation}
 \|\NS(D)-\NS(D')\|_F\leq4.25^5\,\|X_0-X_0'\|_F .
\end{equation}
With $\rho=\|D\|_F$, $\rho'=\|D'\|_F$ and $\epsilon=\epsilon_{\rm NS}$,
\begin{equation}
 X_0-X_0'=\frac{D-D'}{\rho+\epsilon}+\frac{(\rho'-\rho)\,D'}{(\rho+\epsilon)(\rho'+\epsilon)},
\end{equation}
and $\|D'\|_F=\rho'$ and $|\rho-\rho'|\leq\|D-D'\|_F$ give
\begin{equation}
 \|X_0-X_0'\|_F\leq\frac{\|D-D'\|_F}{\rho+\epsilon}+\frac{|\rho-\rho'|}{\rho+\epsilon}\leq\frac{2\|D-D'\|_F}{\rho+\epsilon} .
\end{equation}
Since $\rho\geq\min(\rho,\rho')$, this proves Equation~\eqref{eq:cth-nslip}. For $m>n$, the same argument applies to the transposes.
\end{proof}

The direction $\NS(D)$ is in general not the polar factor of $D$ \citep[Chapter~8]{higham2008functions}, whose nonzero singular values equal one. For example, a matrix of rank one with $\|D\|_F$ much larger than $\epsilon_{\rm NS}$ has $\hat\sigma_1$ close to one, and $P_5(1)=0.6964\ldots$. The results above use the actual values $P_5(\hat\sigma_i)$.

\end{document}